\documentclass{article}

\usepackage[preprint]{neurips_2026}

\usepackage[utf8]{inputenc}
\usepackage[T1]{fontenc}
\usepackage{hyperref}
\usepackage{url}
\usepackage{booktabs}
\usepackage{amsfonts}
\usepackage{amssymb}
\usepackage{amsmath}
\usepackage{amsthm}
\usepackage{nicefrac}
\usepackage{microtype}
\usepackage{xcolor}
\usepackage{bm}
\usepackage{dsfont}
\usepackage{enumitem}
\usepackage{CJKutf8}
\usepackage{setspace}
\usepackage{comment}
\usepackage{algorithm}
\usepackage{latexsym}
\usepackage[flushleft]{threeparttable}
\usepackage{longtable}
\usepackage{algpseudocode}
\usepackage{wrapfig}
\usepackage{lipsum}
\usepackage{makecell}
\usepackage{subcaption}
\usepackage{adjustbox}
\usepackage{tikz}

\usepackage{authblk}

\usetikzlibrary{shapes.geometric, arrows.meta, positioning, fit, backgrounds}

\allowdisplaybreaks

\theoremstyle{plain}
\newtheorem{theorem}{Theorem}
\newtheorem{corollary}{Corollary}
\newtheorem{remark}{Remark}

\usepackage{etoolbox}
\usepackage{tabularx, booktabs, graphicx}

\makeatletter
\newcommand{\ostar}{\mathbin{\mathpalette\make@circled\star}}
\newcommand{\make@circled}[2]{%
  \ooalign{$\m@th#1\smallbigcirc{#1}$\cr\hidewidth$\m@th#1#2$\hidewidth\cr}%
}
\newcommand{\smallbigcirc}[1]{%
  \vcenter{\hbox{\scalebox{0.77778}{$\m@th#1\bigcirc$}}}%
}
\makeatother

\newcommand{\commenting}[1]{}

\renewcommand{\hat}{\widehat}

\newcommand{\Var}[1]{{\operatorname{Var}\left\{#1\right\}}}

\newcommand{\Cov}[2]{{\operatorname{Cov}\left\{#1,#2\right\}}}

\newcommand{\E}[1]{{\mathbb{E}\left\{#1\right\}}}

\newcommand{\ind}[1]{\boldsymbol{1}\left\{#1\right\}}

\ifdef{\see}{\renewcommand{\see}[1]{\text{ (#1)}}}{\newcommand{\see}[1]{\text{ (#1)}}}

\def\boxit#1{\vbox{\hrule\hbox{\vrule\kern6pt\vbox{\kern6pt#1\kern6pt}\kern6pt\vrule}\hrule}}

\newcolumntype{P}[1]{>{\centering\arraybackslash}p{#1}}
\newcolumntype{M}[1]{>{\centering\arraybackslash}m{#1}}
\newcolumntype{L}[1]{>{\raggedright\arraybackslash}m{#1}}

\newcommand{\bbR}{{\mathbb{R}}}

\title{BACON: Budgeted Human Calibration for Modeling and Evaluation with Multiple AI Judges}

\author[1]{{Lei Shi}}
\author[1]{{Anlan Zhang}}
\author[2]{{Rita Lyu}} 
\author[1]{{Zhengmian Hu}}
\author[1]{{Tong Yu}} 
\author[1]{{David Arbour}}
\author[2]{\authorcr {Avi Feller}}
\author[1]{{Saayan Mitra}}
\author[1]{{Ritwik Sinha}} 

\affil[1]{Adobe Research}
\affil[2]{University of California, Berkeley}

\begin{document}

\maketitle

\begin{abstract}
    AI judges are increasingly used to reduce the cost of human evaluation. While they provide a scalable and inexpensive alternative, their outputs can be biased relative to human preferences and highly item-dependent, with substantial variation across judges, tasks, and domains. When practitioners rely on uncalibrated AI evaluations for model ranking, item scoring, or population-level quality reporting, systematic bias can propagate directly into downstream decisions and make the conclusions unreliable. We propose BACON, a four-stage pipeline that combines budgeted human calibration with multiple AI-judge outputs to obtain more accurate annotations. BACON first constructs full-coverage auxiliary features for every item in the evaluation pool, including multi-judge scores, token-level uncertainty statistics, and contextual embeddings. It then obtains human labels for a small sampled subset and trains a cross-fitted outcome model to produce calibrated item-level surrogate predictions. These predictions support two downstream modes: the first is population-level estimation of summary metrics, such as means, quantiles, or other estimands, using an augmented estimating-equation estimator with valid confidence intervals; and the second is individual-level surrogate scoring for item scoring and ranking. Throughout, BACON treats AI judges as auxiliary measurements rather than ground truth: human labels provide the calibration anchor, while AI-derived signals serve as surrogate predictions and improve efficiency. We validate BACON on a variety of tasks and domains. Across settings and labeling budgets, the cross-fitted outcome model improves predictive accuracy and ranking consistency, while the estimating-equation estimator reduces bias and variance relative to raw AI outputs and purely human-label-based methods. These results suggest that BACON provides a practical framework for scalable, statistically grounded evaluation with limited human annotation.
\end{abstract}

\section{Introduction}\label{sec:introduction}

\subsection{Background: Bias in AI-assisted evaluations}
Many important decisions rely on large-scale human evaluations: grading student writing, scoring persuasive essays, and rating the aesthetics, trustworthiness, and usability of websites, among others. Human labels are expensive, especially when the goal is not only to score individual items but also to estimate {population-level} summary statistics, such as the mean quality of a corpus, under a fixed budget. Using generative AI as scalable judges has become a popular practice, providing low-cost scores for large item pools. However, AI-judge outputs can be systematically biased relative to human preferences, and their errors are often item-dependent and model-specific. Below, we give three examples showing that this mismatch appears across tasks and modalities.

\textbf{Case Study 1: Essay scoring with PERSUADE Dataset}. The PERSUADE dataset \citep {crossley2024large} contains over 25,000 argumentative essays from U.S. students in grades 6--12, each assigned a holistic human score from 1 to 6. We prompt several LLMs to score the same essays and compare their outputs to the human labels (prompt in Appendix \ref{app:prompt}). 
Figure~\ref{fig:persuade-score-distribution} compares human and LLM score distributions on PERSUADE. Human scores are more heavy-tailed than the AI scores, and some AI models (e.g. llama-3.1-8b) concentrate around an entirely different peak.

\begin{figure*}[ht!]
    \centering
    \begin{subfigure}[t]{0.4\linewidth}
        \centering
        \includegraphics[width=\linewidth]{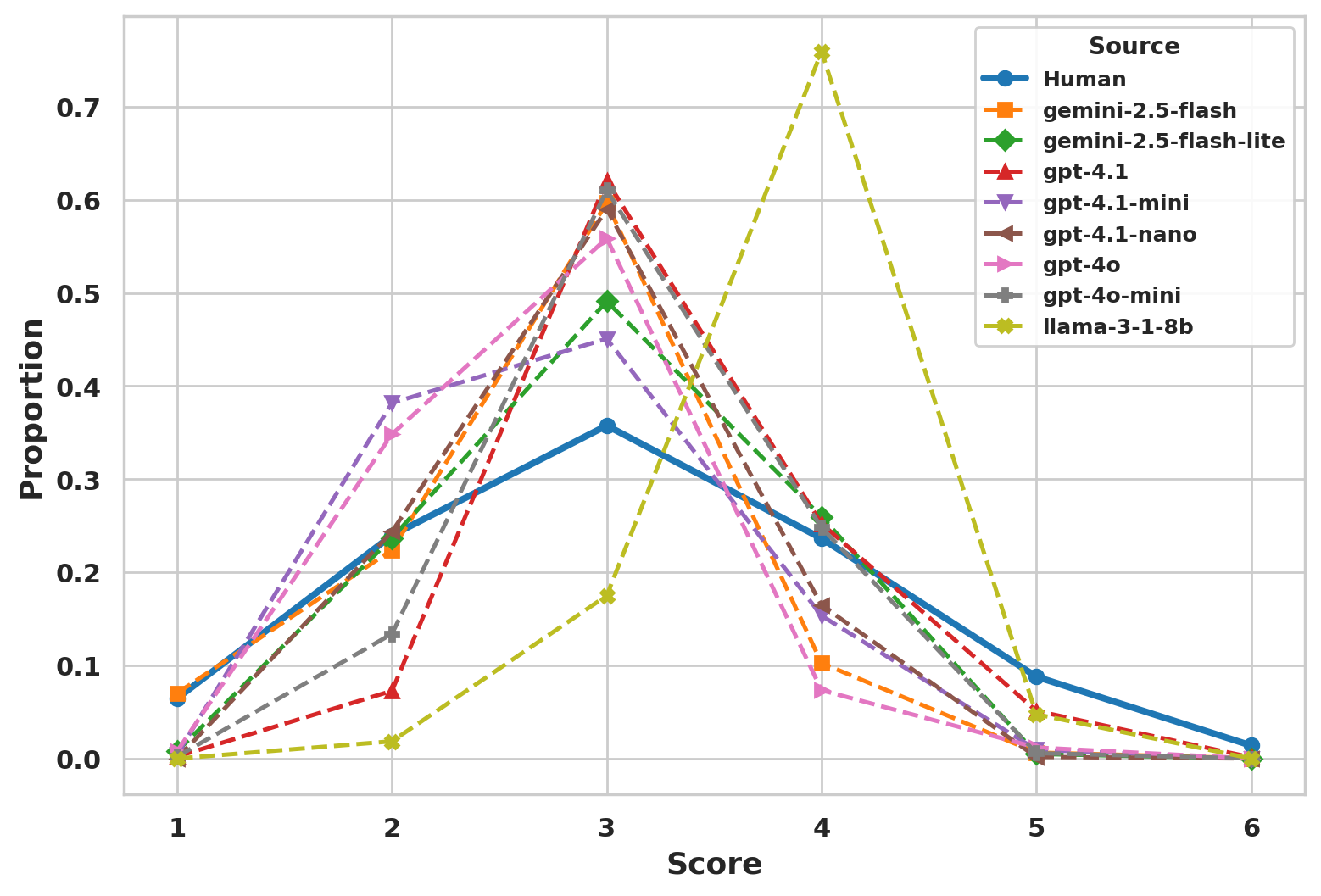}
        \caption{}
        \label{fig:persuade-score-distribution}
    \end{subfigure}
    \begin{subfigure}[t]{0.4\linewidth}
        \centering
        \includegraphics[width=\linewidth]{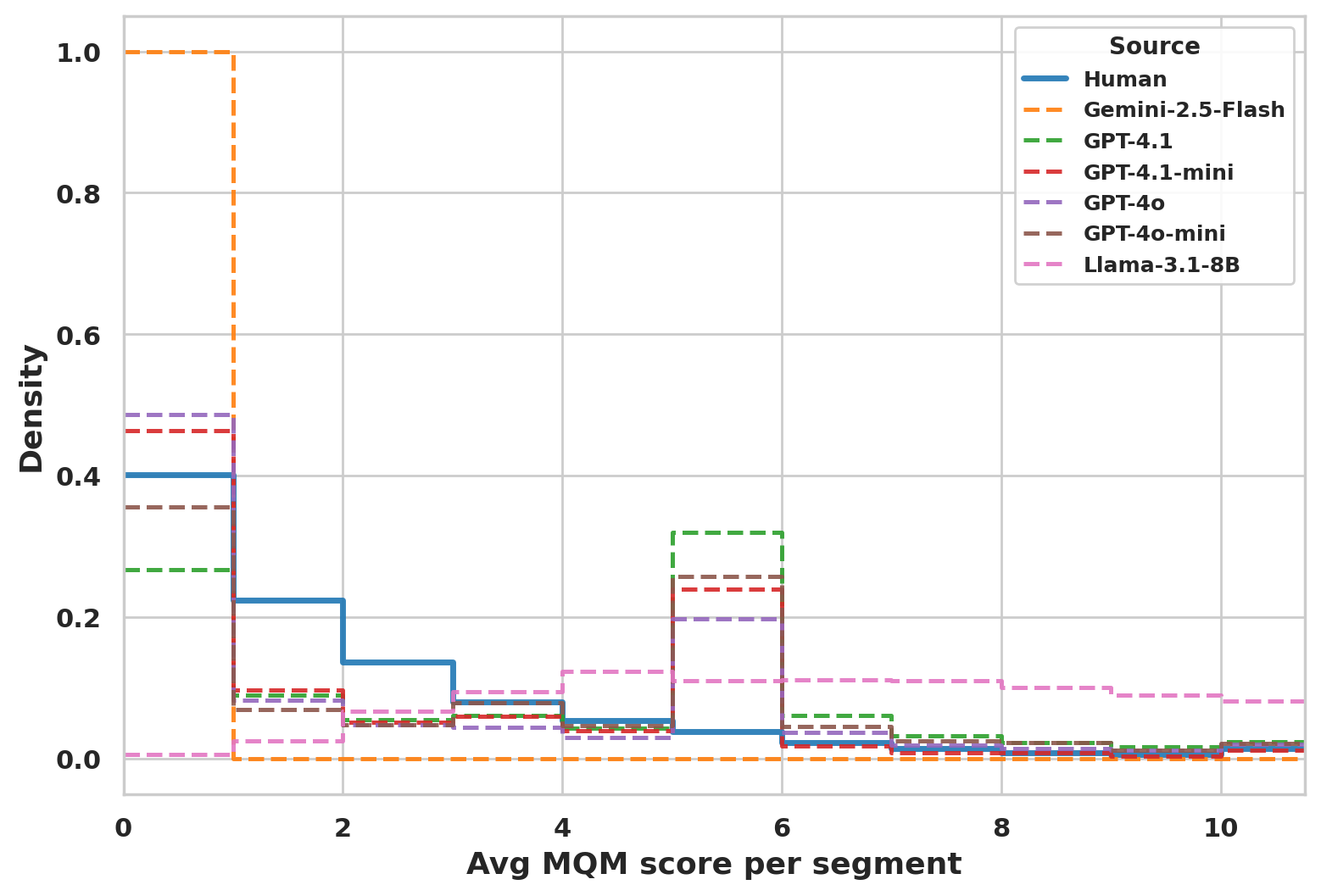} 
        \caption{}
        \label{fig:mqm-alignment}
    \end{subfigure}
    \caption{Score distribution comparisons between human raters and AI judges for PERSUADE and MQM. Panel (a): Grade-10 PERSUADE score proportions for human labels and all 8 LLM judges (scores 1--6). Human shown as solid line; LLMs as dashed lines. Panel (b): MQM segment-level score distributions for human raters and six LLM judges on WMT 2020 en--de. Both datasets show systematic disagreement between human labels and AI judges. }
    \label{fig:score-distributions}
\end{figure*}

\textbf{Case Study 2: Machine translation evaluation}. We also study the WMT 2020 English--German MQM dataset \citep{freitag2021experts}, where expert raters annotate translation errors by category and severity and each segment receives a non-negative MQM score (lower is better). We use six LLM judges (listed in Figure~\ref{fig:mqm-alignment}) to produce MQM-style annotations across ten MT systems, including three human translations. Figure~\ref{fig:mqm-alignment} shows systematic miscalibration: some LLM judges largely underestimate or overestimate error rates (Gemini-2.5-Flash and Llama-3.1-8B), while others demonstrate a two peak distribution that is not observed in human labels.

\textbf{Case Study 3: Web-design scoring}. Finally, we consider visual evaluation of website screenshots, where models must judge perceptual qualities such as aesthetics, trustworthiness, typicality, and usability \citep{miniukovich2023effect}. Figure~\ref{fig:webdesign-unis-score-distribution} shows the score distributions for the Universities subset of the WebDesign dataset and reveals substantial disagreement between human labels and VLM judges.

\begin{figure*}[t]
    \centering
    \includegraphics[width=\linewidth]{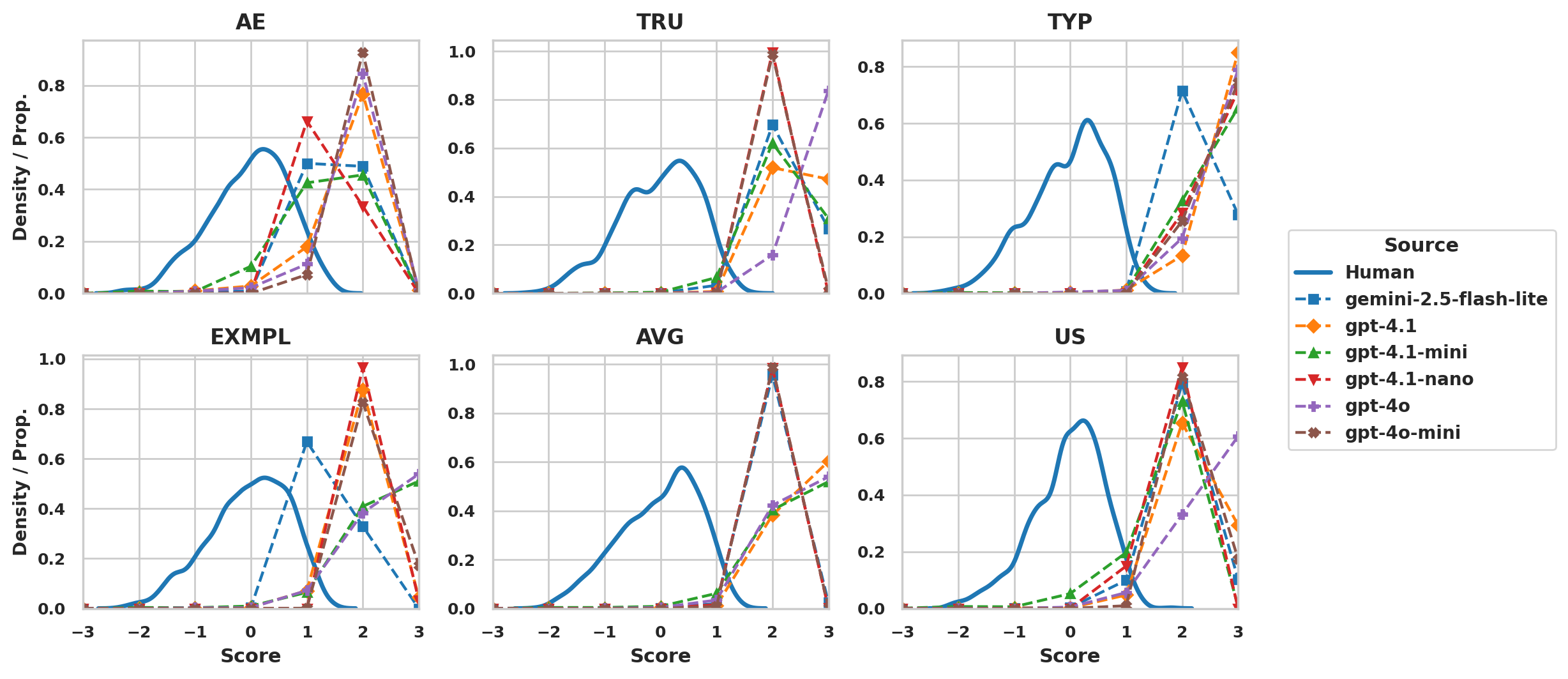}
    \caption{WebDesign Universities score distributions for six perceptual outcomes (scores $-3$ to $3$). Human scores are shown as KDE curve; VLM judges as proportion lines.}
    \label{fig:webdesign-unis-score-distribution}
\end{figure*}

Together, these examples show a common problem: AI scores usually don't target the same measurement as the human labels. If practitioners simply average AI scores or use them to rank items or systems, they are making an implicit and usually unvalidated measurement assumption. We propose BACON to address this problem.

\subsection{BACON framework overview and our contributions}
BACON is an evaluation protocol for making human-aligned claims under limited labels. For a target item population, the evaluator specifies the human outcome and the evaluation report of interest, such as a population mean, system ranking, or item-level score. BACON collects full-coverage auxiliary measurements from multiple AI judges, including scores, uncertainty statistics, and item embeddings. A small portion of samples is then labeled by humans and used to fit a cross-fitted calibration model from AI measurements to the target human outcome. BACON reports calibrated item-level surrogate scores and aggregate human-aligned estimates with uncertainty intervals. Throughout the protocol, AI judges are treated as auxiliary measurements rather than replacements for human labels: human calibration anchors the validity of the final claim, while AI measurements improve efficiency when they are predictive.

To summarize, we make the following contribution:

(1) \textbf{Human-calibrated AI-judge evaluation protocol.}
We formulate AI-assisted evaluation as a budgeted human-calibration problem in which multiple AI judges are treated as auxiliary measurements, not ground truth. AI judge scores provide full-coverage outcome predictions, and a small uniform human sample supplies a bias correction, decoupling estimation accuracy from the quality of the outcome model.

(2) \textbf{Multi-judge calibration recipes.}
BACON distinguishes two evaluation tasks: empirically calibrated item-level surrogate scoring and statistically valid summary metric estimation.
For item-level evaluation, we propose building outcome models from multiple judge scores, uncertainty summaries, and context embeddings, and discuss practical model choices for different type of outcomes such as ordinal, zero-inflated, continuous, and categorical evaluation outcomes. For summary statistics report, BACON follows an estimating equation-based statistical framework to deliver estimators that are consistent for human assessment and provide valid uncertainty quantification through confidence intervals.

(3) \textbf{Cross-domain empirical study of AI-judge calibration.}
We validate BACON across three real-world datasets spanning multiple tasks and modalities, including PERSUADE essay scoring, MQM machine translation evaluation, and WebDesign web-design scoring. The results demonstrate substantial improvements in judge-model predictive accuracy for individual item quality and the estimation efficiency for summary metrics over several baselines. The code for replication is hosted in the following GitHub repo (anonymized for review): \url{https://github.com/diaryofnewton/bacon-calibration}.

\subsection{Related Works}

\textbf{LLM-as-a-judge.}
\citet{zheng2023llmjudge} establish the canonical LLM-as-judge paradigm via MT-Bench and Chatbot Arena; \citet{liu2023geval} extend it with chain-of-thought prompting for text-generation evaluation; and \citet{dubois2024alpacaeval} develop AlpacaEval for instruction-following comparison.
Despite their utility, LLM judges exhibit systematic biases in terms of position, verbosity, and self-enhancement \citep{wang2023faireval}, that distort rankings; surveys \citep{li2024llmjudge, gu2024llmjudge} catalog these failure modes.
BACON highlights such biases as a starting point and provides a principled calibration framework to correct them using a small sample of human labeled data.

\textbf{Multi-judge aggregation.}
Self-consistency \citep{wang2023selfconsistency} and multi-agent debate \citep{du2024multiagentdebate, liang2024divergent} use ensembles of reasoning paths or models to improve reliability.
In the evaluation setting, ChatEval \citep{chan2023chateval}, jury-style panels \citep{verga2024jury}, and structured multi-agent collaboration \citep{qian2025multijudge} show that aggregating diverse judges reduces individual biases.
BACON differs: rather than debate or ad-hoc ensembling, we treat judge scores as features in a supervised outcome model fit against human labels, yielding statistically grounded estimates.

\textbf{Prediction-powered inference (PPI) and semi-supervised estimation.}
\citet{angelopoulos2023ppi} introduce PPI, using model predictions plus a small labeled set to form valid confidence intervals; \citet{angelopoulos2023ppipp} extend this to PPI++ with adaptive weighting; \citet{zrnic2024crossppi} remove the held-out requirement via cross-fitting.
These ideas root in classical survey regression estimators \citep{sarndal1992model, cassel1976generalized} and semi-supervised inference \citep{zhang2019semisup, cai2018semisup}. In the application of summary metrics estimation, BACON follows directly this line of work. For the special case of estimating a population mean, our estimator is mathematically equivalent to a PPI / regression-adjusted estimator with a learned prediction function. The contribution of BACON is therefore not a new mean estimator. Instead, BACON addresses a problem left open by generic PPI formulations: how to construct reliable auxiliary predictions in AI-assisted evaluation when the available predictors are multiple biased AI judges, uncertainty statistics, as well as item embeddings, and how to report both aggregate inferential quantities and item-level surrogate scores without treating AI judgments as ground truth, which serves as the protocol layer between LLM-as-judge and PPI.


\textbf{Training and fine-tuning LLM judges.}
JudgeLM \citep{zhu2023judgelm}, Prometheus \citep{kim2024prometheus, kim2024prometheus2}, and PandaLM \citep{wang2024pandalm} fine-tune open-source models as scalable evaluators, while \citet{li2023generativejudge} produce evaluative rationales rather than scalar scores.
However, \citet{huang2025finetunedjudge} find that fine-tuned judges often underperform GPT-4 out-of-domain.
BACON is training-free: off-the-shelf judge outputs feed a lightweight calibration layer, requiring no judge-specific training data and supporting closed-source models.

\section{BACON Framework}

\subsection{Overview and notation setup}
We present a training-free pipeline for AI-assisted evaluation. BACON integrates three components to mitigate bias in AI judge outputs: (1) aggregating predictions from multiple judges to capture cross-judge variability; (2) leveraging contextual features to debias judge predictions; and (3) calibrating AI evaluations using a small sample of human-labeled data.

We study AI-assisted modeling and evaluation under limited human labeling budgets. Suppose we have a pool of items, each associated with a scalar or vector-valued human outcome. Our goal is to estimate population-level performance and improve item-level prediction by combining AI judge outputs in a principled way, with a small amount of human labels for calibration. 

Formally, for item $i \in \{1,\dots,N\}$, let $y_i$ denote the human score (or one outcome component), and let $\mathbf{\ell}_i \in \mathbb{R}^M$ collect $M$ model-judge scores. We also construct contextual embedding features $e_i$ (e.g., text embeddings for essays and image embeddings for screenshots). We then fit predictive models for $y_i$ using judge scores, uncertainty features, and embedding features to capture their relationship with human evaluations.

\subsection{Budgeted Evaluation Framework}

The proposed evaluation framework is as follows; Figure~\ref{fig:bacon-pipeline} gives an overview.

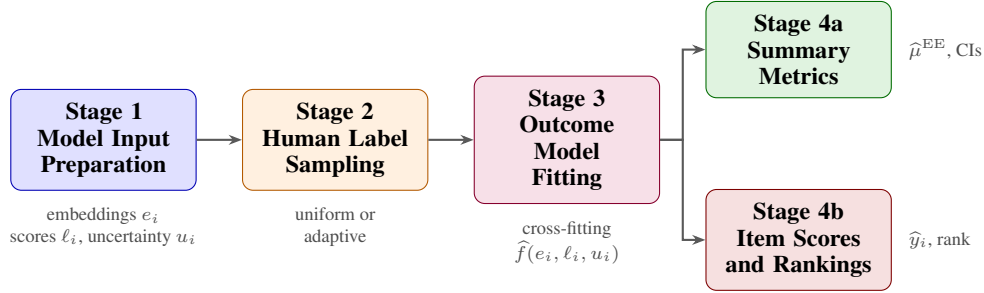
\begin{figure}[ht]
\centering
\begin{tikzpicture}[
  node distance = 4mm and 6mm,
  box/.style    = {rectangle, rounded corners=4pt, draw=#1!70!black,
                   fill=#1!12, text width=2.1cm, align=center,
                   font=\small\bfseries, inner sep=5pt, minimum height=1.0cm},
  arr/.style    = {-{Stealth[length=5pt]}, thick, gray!70!black},
  label/.style  = {font=\scriptsize, text=gray!60!black, align=center},
]

\node[box=blue]   (s1) {Stage 1\\Model Input\\Preparation};
\node[box=orange, right=of s1] (s2) {Stage 2\\Human Label\\Sampling};
\node[box=purple, right=of s2] (s3) {Stage 3\\Outcome Model\\Fitting};

\node[box=green!60!black, above right=-3mm and 6mm of s3] (d1)
      {Stage 4a\\Summary\\Metrics};
\node[box=red!70!black,   below right=-2mm and 6mm of s3] (d2)
      {Stage 4b\\Item Scores\\and Rankings};

\draw[arr] (s1) -- (s2);
\draw[arr] (s2) -- (s3);
\draw[arr] (s3.east) -- ++(3mm,0) |- (d1.west);
\draw[arr] (s3.east) -- ++(3mm,0) |- (d2.west);

\node[label, below=1mm of s1] {embeddings $e_i$\\scores $\ell_i$, uncertainty $u_i$};
\node[label, below=1mm of s2] {uniform or\\adaptive};
\node[label, below=1mm of s3] {cross-fitting\\$\hat{f}(e_i,\ell_i,u_i)$};
\node[label, right=1mm of d1] {$\hat\mu^{\mathrm{EE}}$, CIs};
\node[label, right=1mm of d2] {$\hat y_i$, rank};

\end{tikzpicture}
\caption{BACON evaluation pipeline. Stage 1 prepares inputs (context features and AI jury scores) for all items; Stage 2 collects a small human label sample; Stage 3 fits the outcome model with cross-fitting; Stage 4 routes to summary-metric estimation or individual scoring and ranking.}
\label{fig:bacon-pipeline}
\end{figure}

(1) Stage 1: Model input preparation. (i) Context features extraction.  For each item in the evaluation pool, we extract some context features. This includes the content embeddings, summary statistics, emotion and psychological attributes, among others \citep{mozer2025more}. (ii) AI Jury Scoring. We use multiple AI judges to score all the items in the pool. We can record auxiliary information such as the token entropy and perplexity. We can also run the AI judges for multiple rounds. 

(2) Stage 2: Human Label Sampling. We sample a small number of items to get human labels. The basic approach is to sample items uniformly at random, given a fixed budget. Also, we can apply adaptive sampling strategies to improve the sampling efficiency. We leave an detailed discussion of adaptive sampling strategies to Appendix \ref{app:sampling}.

(3) Stage 3: Outcome Model Fitting. We fit an outcome model to the human labels and the AI scores. For this part, we  use cross-fitting to obtain out-of-fold predictions to avoid overfitting. 

(4) Stage 4: Route the model for downstream tasks. Here we list two representative tasks to use the scores:
(i) Estimate summary metrics. We build regression-adjusted estimators to estimate aggregated metrics such as the population average of human scores, as well as providing uncertainty quantifications through confidence intervals. (ii) Get individual item predictions and ranks. We can use the model to predict surrogate human scores and rank the items. The surrogate scores can be used to flag low quality items. 

\subsection{Details of the BACON Framework}

\textbf{Joint modeling of context features and multi-judge evaluations. }
Overall, we try to fit a predictive model for $y_i$ from the context features $e_i$ and the judge evaluations $\ell_i$. Moreover, we can incorporate additional features $u_i$ which encodes the uncertainty of the judge evaluations. We propose to fit a joint model for $y_i$ from the context features $e_i$, the judge evaluations $\ell_i$ with uncertainty features $u_i$:
\begin{align*}
    \hat{y}_i = f(e_i, \ell_i, u_i),
\end{align*}
where $f$ is a function that we need to learn from data. We will discuss the choice of $f$ later. Intuitively, we hope the function $f$ can serve as a meta-judge that aggregates the single-judge evaluations and steer the prediction towards the human preferences. 

Embedding features $e_i$ capture item-level effects, such as topic, style, visual composition, that are not explained by judge scores alone, acting as fixed-effect controls.
Multiple judge scores $\ell_i$ are combined via model fitting rather than ad-hoc majority voting or averaging, allowing the outcome model to down-weight biased or unreliable judges.
Uncertainty features $u_i$ encode judge confidence: token-level entropy (available when log-probabilities are accessible) or sample-level variance from multiple judge runs (applicable to any model).

\textbf{Example models.} The choice of $f$ depends on the score pattern of the outcome. Table~\ref{tab:example-models} summarizes several instantiations; all share the same feature inputs $(e_i, \ell_i, u_i)$.
\begin{table}[ht]
\centering
\caption{Example outcome models in BACON. All models use embeddings $e_i$, judge scores $\ell_i$, and uncertainty features $u_i$ as inputs.}
\label{tab:example-models}
\small
\begin{tabular}{p{2.5cm} p{3.0cm} p{6.8cm}}
\toprule
\textbf{Score pattern} & \textbf{Example model} & \textbf{Key idea} \\
\midrule
Real-valued &
  Ridge regression &
  Linear predictor $\hat{y}_i = \alpha_0 + \alpha_e^\top e_i + \alpha_\ell^\top \ell_i + \alpha_u^\top u_i$; $\ell_2$ penalty controls judge collinearity. \\[4pt]
Non-negative, zero-inflated &
  Hurdle model (logistic gate $+$ conditional ridge) &
  Gate $\Pr(z_i{=}1)$ predicts whether $y_i{>}0$; conditional ridge predicts $\hat{y}_i^+$ on positive examples; final prediction $\hat{y}_i = \Pr(z_i{=}1)\cdot\hat{y}_i^+$. \\[4pt]
Ordinal &
  Proportional-odds (ordered logit) &
  Linear predictor $\eta_i = \alpha_0 + \alpha_e^\top e_i + \alpha_\ell^\top \ell_i + \alpha_u^\top u_i$ with $K{-}1$ ordered cutpoints $\theta_1 {<} \cdots {<} \theta_{K-1}$; $\Pr(Y_i{\leq}k) = \sigma(\theta_k - \eta_i)$. \\[4pt]
Categorical &
  Multinomial regression &
  $\Pr(Y_i{=}k) = \exp(\eta_i^{(k)}) / \sum_{j=1}^K \exp(\eta_i^{(j)})$, where $\eta_i^{(k)} = \alpha_0^{(k)} + \alpha_e^{(k)\top} e_i + \alpha_\ell^{(k)\top} \ell_i + \alpha_u^{(k)\top} u_i$. \\
\bottomrule
\end{tabular}
\end{table}

\textbf{Cross-fitting for outcome model fitting.}
A key requirement is that the outcome model $\hat{f}$ must not be fitted on the same sampled observations whose residuals are used for bias correction. If the same data is used for both fitting and prediction, the in-sample fit can be overly optimistic---a model that memorises its training labels will produce near-zero residuals, inflating bias for complex, low-regularization models. Cross-fitting \citep{chernozhukov2018double} resolves this by partitioning the sampled set $\mathcal{S} = \{i : S_i = 1\}$ into $K$ folds $\mathcal{S}_1, \ldots, \mathcal{S}_K$. For each fold $k$, the outcome model $\hat{f}^{(-k)}$ is trained on all sampled observations \emph{except} fold $k$, and its predictions on fold $k$ are used as the out-of-fold (OOF) predictions. Formally, for $i \in \mathcal{S}_k$:
\[
    \hat{y}_i = \hat{f}^{(-k)}(e_i, \ell_i, u_i), \qquad \hat{f}^{(-k)} = \operatorname{fit}\!\left(\{(e_j, \ell_j, u_j, Y_j) : j \in \mathcal{S} \setminus \mathcal{S}_k\}\right).
\]
Then we fit the an aggregated model $\hat{f}$ on the full sample $S$ and use it to predict the labels for items that are not in $S$. Through this formulation, the predictions $\hat{y}_i$'s cover all $N$ items. Because $\hat{f}^{(-k)}$ is fitted on a held-out split, the residuals $Y_i - \hat{y}_i$ are honest as they reflect genuine generalisation error rather than in-sample fit. This prevents overfit models from artificially suppressing the bias-correction variance, and makes the variance estimate more stable when the outcome model is complex. In practice, we use $K{=}5$ folds, each trained on 80\% of the sampled data, which balances the bias--variance trade-off for the OOF predictions.

\textbf{Summary metrics estimation with uncertainty quantification.} Summary metrics include population mean, median/quantiles, variance, etc. We can unify the estimation of common summary metrics using an estimating equation framework. Let $Y$ be the observed human scores (with \texttt{NA} if no human label is collected), $S$ be the indicator variable indicating whether a human label is collected, $X$ be some measurements for the context, such as the embedding features, the judge scores and uncertainty summary metrics, etc. Below we provide an estimating equation framework \citep{hardin2002generalized}. 

Suppose we are interested in a target parameter $\theta^\star = \theta(\mathbb{P}_{Y, S, X})$, which is the solution to the following estimating equation:
$
\E{m(Y, S, X;\theta^\star, \gamma^\star)} = 0,
$ 
where $m(Y, S, X;\theta, \gamma)$ is the estimating equation, and $\gamma^\star$ is a nuisance parameter. For example, for the mean and $\tau$-th quantile, we have 
\[
m_{\text{mean}}(Y, S, X;\theta, \gamma) = \frac{SY}{\pi(X)} - \theta,
\quad \text{and} \quad
m_{\text{quantile}}(Y, S, X;\theta, \gamma) = \frac{S \ind{Y \le \theta}}{\pi(X)} - \tau,
\]
respectively. Here $\pi(X)$ is the labeling propensity and serves as a given nuisance parameter $\gamma^\star$. In practice, we can apply many classical sampling schemes to collect the human labels, which corresponds to different choices of $\pi(X)$. For example, we can draw items uniformly at random, or perform stratified sampling based on some context features $X$. Alternatively, we may learn the labeling propensity from some pilot data using a propensity model. See Appendix \ref{app:sampling} for more details.

It is usually desirable to incorporate a regression adjustment term to the estimating equation to improve the efficiency of the estimator. For the above examples, we can augment the estimating equation as:
\begin{gather*}
    m_{\text{mean}}(Y, S, X;\theta, \gamma) = \frac{S(Y - \mu(X))}{\pi(X)} + \mu(X) - \theta,
    \\
    m_{\text{quantile}}(Y, S, X;\theta, \gamma) = \frac{S (\ind{Y \le \theta} - \phi(X;\theta))}{\pi(X)} + \phi(X;\theta) - \tau,
\end{gather*}
respectively. Here $\mu(X)$ is a regression adjustment for the human scores and $\phi(X;\theta)$ is a regression adjustment for the quantile. We can then estimate the target parameter $\theta^\star$ by solving the following estimating equation: $\mathbb{E}_N(m(Y, S, X; \hat{\theta}, \hat{\gamma})) = N^{-1}\sum_{i=1}^N m(Y_i, S_i, X_i; \hat{\theta}, \hat{\gamma}) = 0$, which we call as the augmented estimating-equation (AEE) estimator. Under some regularity conditions (Appendix \ref{app:ee}), the estimator $\hat{\theta}$ is consistent and asymptotically normal:
$
\sqrt{N}(\hat{\theta}-\theta^\star)\rightsquigarrow \mathcal{N}(0,\Sigma_\theta).
$
The asymptotic covariance takes the sandwich form $\Sigma_\theta = A^{-1} B A^{-T}$, where $A$ and $B$ are given by
\[
A = \E{\nabla_\theta m(Y,S,X;\theta^\star,\gamma^\star)}, \quad \text{and} \quad
B = \E{m(Y,S,X;\theta^\star,\gamma^\star)\,m(Y,S,X;\theta^\star,\gamma^\star)^\top}.
\]
Accordingly, a plug-in sandwich estimator of $\textup{Var}\{\hat{\theta}\}$ is $\widehat{v} = \frac{1}{n}\hat{A}^{-1}\hat{B}\hat{A}^{-T}$, with $\hat{A}$ and $\hat{B}$ given by
\[
\hat{A} = \mathbb{E}_N{\nabla_\theta m(Y,S,X;\hat{\theta},\hat{\gamma})}, \quad \text{and} \quad
\hat{B} = \mathbb{E}_N{m(Y,S,X;\hat{\theta},\hat{\gamma})\,m(Y,S,X;\hat{\theta},\hat{\gamma})^\top}.
\]
Then we can construct a confidence interval for $\theta^\star$ as $\hat{\theta} \pm z_{\alpha/2} \sqrt{\widehat{v}}$ with confidence level $1-\alpha$. When the target is the population mean, the AEE estimator is equivalent to the inverse-propensity weighted estimator upon regression adjustment in survey sampling \citep{sarndal1992model,cassel1976generalized} or predictive powered inference \citep{angelopoulos2023ppi}. For a rigorous technical presentation, please refer to Appendix \ref{app:ee}.

\section{Real-World Results}
\subsection{Setup}

\textbf{Datasets and Outcomes. }
(i) \textit{PERSUADE.} We use the PERSUADE training corpus with essay-level holistic labels \citep{crossley2024large}. In particular, main experiments target grade-10 essays. (ii) \textit{Machine translation.} We use the WMT 2020 English--German MQM (Multidimensional Quality Metrics) dataset \citep{freitag2021experts}, where expert human raters annotate translation errors by category and severity. Each segment receives a non-negative MQM score (lower is better; zero means no errors detected). (iii) \textit{WebDesign.} We use screenshot datasets from multiple universities \citep{miniukovich2023effect}. For each screenshot, we model six human-rated outcomes: Aesthetics, Trustworthiness, Typicality, Exemplar Goodness, Family Resemblance, and Usability. Due to space limit, we defer the WebDesign results to Appendix \ref{app:webdesign}.

\textbf{Embedding Features. }
(i) \textit{Text embeddings.} Essays/Machine translation segments are embedded using \texttt{text-embedding-3-large}\footnote{\url{https://developers.openai.com/api/docs/models/text-embedding-3-large}} from OpenAI. (ii) \textit{Image embeddings.} Screenshots are embedded with SigLIP \citep{zhai2023sigmoid}. To stabilize downstream linear models, we apply PCA to high-dimensional embeddings, retaining a tunable number of principal components.  

\textbf{Baselines and general metrics. }
We compare BACON against both uncalibrated AI-judge baselines and calibrated outputs. The uncalibrated baselines use raw judge scores directly, including single-judge scores and averages across judges. 
For calibrated models, to isolate the value of different signals, we evaluate calibrated models using only judge scores and uncertainty features, only embedding features, or the full joint feature set. For item-level prediction, we report $R^2$, root mean squared error (RMSE), and Spearman correlation. For summary-metric estimation, we report bias, standard deviation, and RMSE across repeated sampling trials. Monte Carlo experiments are run on a server with an Intel Xeon Platinum 8275CL CPU (96 cores, 3.00 GHz, 1.1 TB RAM). 

\subsection{PERSUADE}
\textbf{Oracle performance}. Table~\ref{tab:persuade-judge-ablation} reports out-of-fold judge performance when fitting the outcome model on grade-10 essays under six configurations with all human labels revealed, in order to evaluate the oracle performance of the model. We use an ordinal regression model. The full 8-judge model combined with embeddings achieves the best performance. Single-judge results reveal a slight performance drop with embeddings, and a large performance drop without embeddings, indicating that embedding features carry substantial predictive signal. As uncalibrated baselines, simply averaging all judges' raw scores (no outcome model) yields $R^2{=}0.361$, and averaging per-judge raw metrics yields $R^2{=}0.150$---substantially below all calibrated configurations. 

\begin{table*}[ht!]
    \centering
    \caption{PERSUADE judge-model performance (5-fold OOF CV,  ordinal regression). Lower RMSE is better; higher $R^2$ and Spearman $\rho$ are better.}
    \label{tab:persuade-judge-ablation}
    \begin{tabular}{lcccc}
        \toprule
        Configuration & Type & $R^2$ & RMSE & Spearman $\rho$ \\
        \midrule
        Eight judges + embedding & Calibrated & 0.699 & 0.602 & 0.831 \\
        Eight judges (no embedding) & Calibrated & 0.399 & 0.851 & 0.619 \\
        \midrule
        Single judge + embedding & Calibrated & 0.683 & 0.618 & 0.822 \\
        Single judge avg. (no embedding) & Calibrated & 0.269 & 0.938 & 0.460 \\
        \midrule
        Embedding-only (ordinal) & Calibrated & 0.674 & 0.627 & 0.816 \\
        Intercept-only constant & Calibrated & $-$0.000 & 1.098 & $-$0.020 \\
        \midrule
        Raw avg.\ of all judges & Uncalibrated & 0.361 & 0.877 & 0.601 \\
        Raw single-judge avg. & Uncalibrated & 0.150 & 1.007 & 0.502 \\
        \bottomrule
    \end{tabular}
\end{table*}

\textbf{Simulation on summary metrics.}
We evaluate how well several configurations support the bias and efficiency claim of the AEE estimator under a uniform sampling budget $\eta \in \{0.05, \ldots, 0.30\}$. In each Monte Carlo trial, a fraction $\eta$ of grade-10 essays is sampled uniformly at random. We run 200 independent trials per budget and report estimator bias, standard deviation (SD), and RMSE across trials.Figure~\ref{fig:persuade-dr-bias-sd-rmse} shows the AEE estimator quality as a function of $\eta$. All calibrated methods are approximately unbiased when the sampling budget is appropriate, while the uncalibrated baselines are biased. The full hybrid model mostly achieves the lowest SD and RMSE across budgets, demonstrating that a better outcome model directly translates to a more efficient estimator. Across all budgets, the reduction in estimator variance from the full model over the intercept-only baseline reflects the predictive $R^2$ gains seen in Table~\ref{tab:persuade-judge-ablation}.

\begin{figure*}[t]
    \centering
    \begin{subfigure}{\linewidth}
        \centering
        \includegraphics[width=0.8\linewidth]{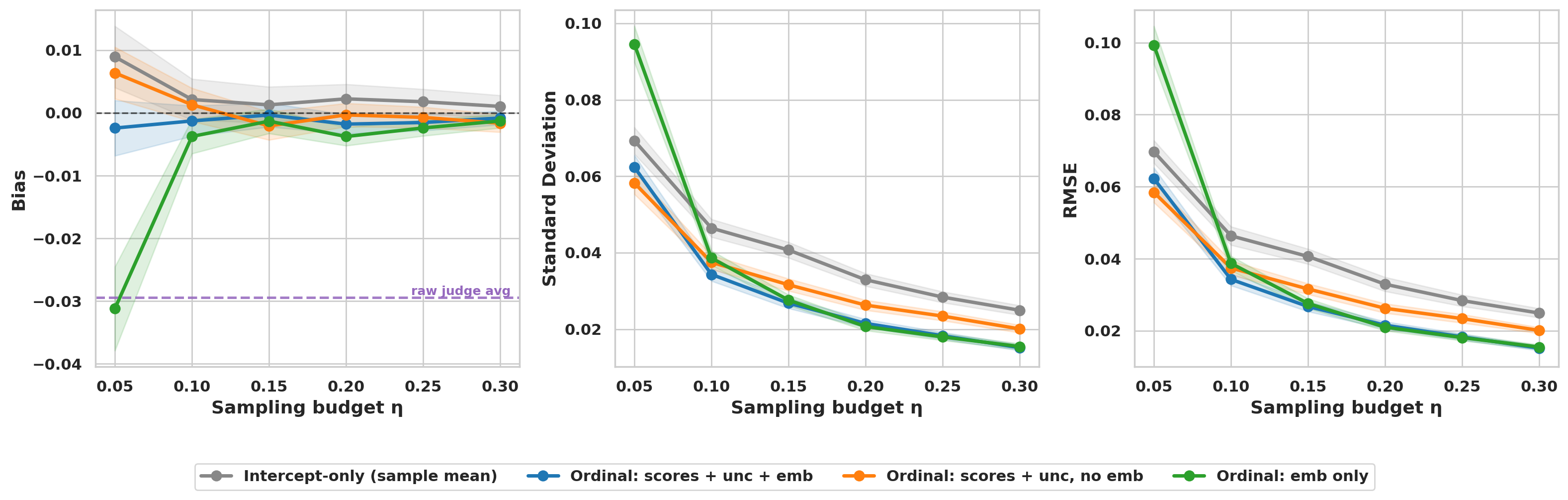}
        \caption{Estimating-equation estimator quality}
        \label{fig:persuade-dr-bias-sd-rmse}
    \end{subfigure}
    \vspace{0.5em}
    \begin{subfigure}{\linewidth}
        \centering
        \includegraphics[width=0.8\linewidth]{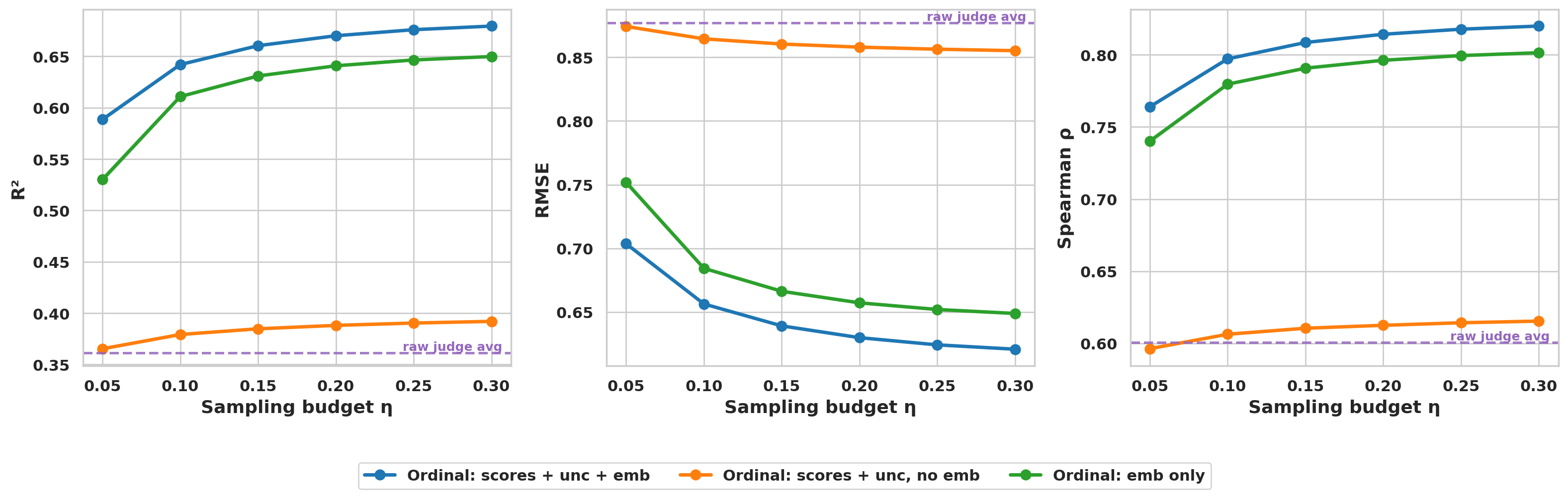}
        \caption{Outcome model quality vs. budget $\eta$.}
        \label{fig:persuade-dr-outcome-quality}
    \end{subfigure}
    \caption{PERSUADE experimental results. Top: bias, standard deviation, and RMSE as a function of $\eta$ (200 MC trials) for mean estimation. Bottom: Outcome model quality vs.\ $\eta$: out-of-fold $R^2$, RMSE, and Spearman $\rho$ fitted on the sampled essays.}
    \label{fig:persuade-dr-combined}
\end{figure*}

\textbf{Outcome model quality vs.\ budget.}
A key question is how well the outcome model itself can be estimated from the sampled essays. Figure~\ref{fig:persuade-dr-outcome-quality} traces out-of-fold $R^2$, RMSE, and Spearman $\rho$ of the judge model as $\eta$ increases, using only the sampled subset for fitting. At the smallest budget ($\eta{=}0.05$, roughly 230 essays), the full model already achieves meaningful predictive performance; quality improves steadily as more essays are observed. This confirms that the joint model provides sufficient signal for reliable outcome modelling even at modest sampling rates, supporting the practical viability of BACON in annotation-scarce settings.

\subsection{Machine translation}

We apply the BACON framework to machine translation quality evaluation using the WMT 2020 English--German MQM dataset.



\textbf{Simulation on summary metrics.}
We evaluate each outcome model as a plug-in for the AEE estimator under uniform sampling budgets $\eta \in \{0.05, \ldots, 0.30\}$ across all ten MT systems. In each of 200 Monte Carlo trials, a fraction $\eta$ of segments per system is sampled uniformly; 5-fold cross-fitting with winsorised LLM score features is used in the evaluation. Figure~\ref{fig:mqm-dr-bias} reports estimator bias, SD, and RMSE across trials.

All methods remain approximately unbiased throughout (max $|\text{bias}| < 0.02$), confirming that the AEE estimator corrects for outcome-model misspecification. The hurdle models (both full and no-emb variants) reduce estimator SD relative to the intercept-only baseline at every budget. The embedding-only hurdle is slightly less efficient than the full model at larger budgets, suggesting that LLM judge scores contribute residual predictive signal beyond text embeddings alone. Besides, all calibrated methods maintain Bonferroni-corrected CI coverage at or above the nominal 95\% level throughout, which suggests the validity of the uncertainty quantification.

\begin{figure*}[t]
    \centering
    \includegraphics[width=1.0\linewidth]{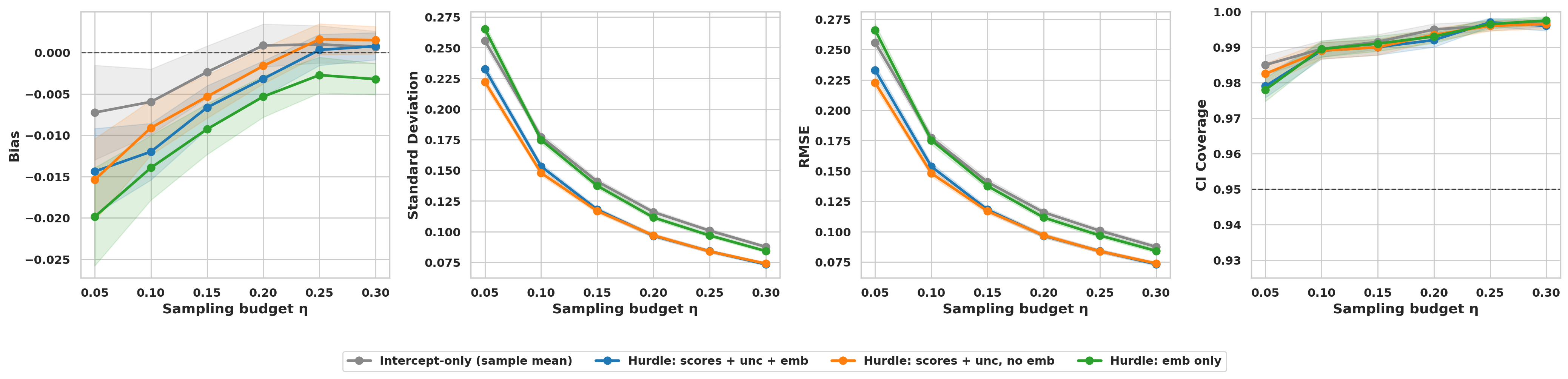}
    \caption{MQM AEE estimator quality vs.\ sampling budget $\eta$ in terms of bias, SD, and RMSE and  Bonferroni-corrected 95\% CI coverage.}
    \label{fig:mqm-dr-bias}
\end{figure*}

\textbf{Translation system comparison under budgeted sampling.}
One key practical goal is correctly ranking the MT systems by quality given limited human labels. Figure~\ref{fig:mqm-ranking-quality} reports Spearman $\rho$, Kendall $\tau$, and mean absolute rank error (MAE) of rankings derived from AEE estimates vs.\ the true system ranking under budgets. All methods produce increasingly accurate rankings as $\eta$ grows. For this goal, hurdle models with both full and multiple-judge-only configurations achieve strong performance at only small budgets. This suggests that system ranking can be achieved with limited human label and proper calibration.

\begin{figure*}[t]
    \centering
    \includegraphics[width=0.8\linewidth]{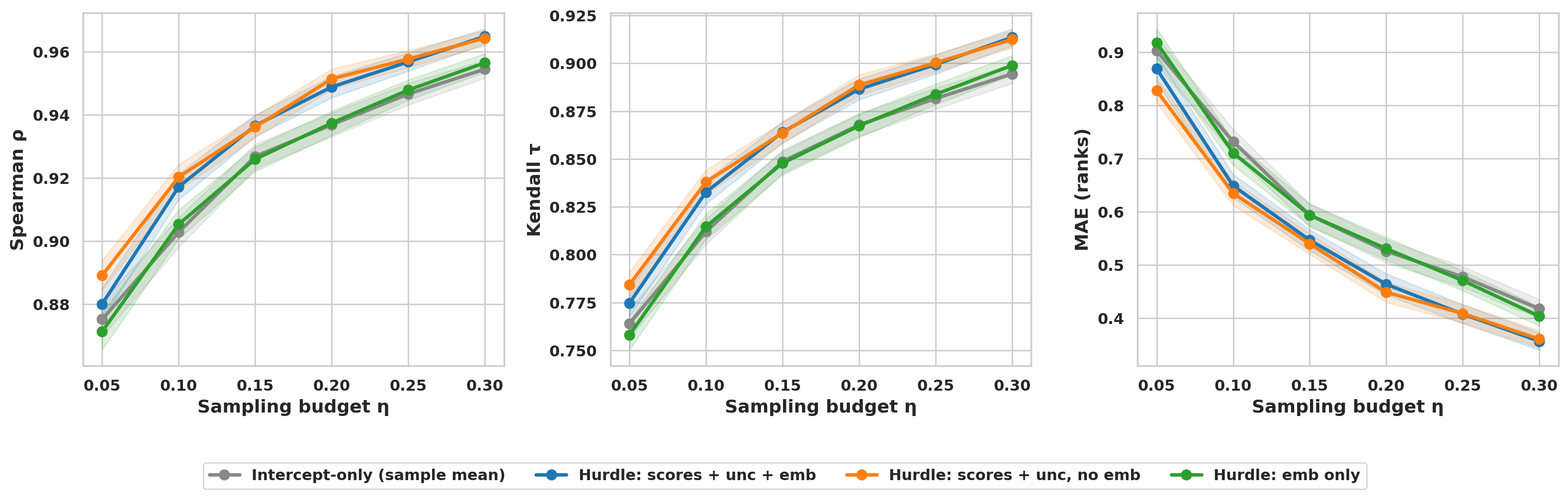}
    \caption{MQM system ranking quality vs.\ sampling budget $\eta$ (200 MC trials, 5-fold cross-fitting). Spearman $\rho$, Kendall $\tau$, and mean absolute rank error (MAE) comparing rankings estimated from AEE to the true system ranking. Hurdle models improve over the intercept baseline at all budgets.}
    \label{fig:mqm-ranking-quality}
\end{figure*}

\textbf{Additional evaluation results.} Section \ref{app:mqm} provides more details on the MQM evaluation, including the oracle performance of the outcome model and the individual item score quality.

\section{Discussion}

\textbf{What BACON adds.}
BACON is a pipeline that aims to provide better guidance for the use of AI judges in evaluation. The item-wise labeling part is closely related to LLM-as-a-judge approaches and LLM collaboration literature, while aggregated estimators are closely related to regression-adjusted survey estimators and PPI-style estimators. The contribution is to bridge these practices in an unified framework and make them operational. Our experiments suggest that this calibration layer matters: raw judge averages can be misaligned with human labels, while calibrated outcome models mitigate the bias by combining the auxiliary signals from multiple sources.

\textbf{Limitations.}
BACON's two output modes have different guarantees. Aggregate reports, such as corpus means or system-level averages, are anchored by human labels and inherit uncertainty quantification under the sampling assumptions. Item-level BACON scores are only calibrated surrogate predictions: they are useful for triage, ranking, and auditing, but should not be interpreted as unbiased measurements for each item. Moreover, calibration may degrade under distribution shift in items, prompts, model versions, or rating rubrics; practical deployment requires periodic recalibration, version logging, and residual monitoring on fresh human labels.

\textbf{Future work.}
Natural extensions include video evaluation, mobile UI assessment, and agentic task evaluation. Future work should also develop stronger item-level uncertainty tools, such as conformal prediction sets, and adaptive sampling strategies that improve efficiency without creating extreme propensities or hiding subgroup failures.

\bibliographystyle{apalike}
\bibliography{sample-base}

@article{miniukovich2023effect,
  title={The effect of prototypicality on webpage aesthetics, usability, and trustworthiness},
  author={Miniukovich, Aliaksei and Figl, Kathrin},
  journal={International Journal of Human-Computer Studies},
  volume={179},
  pages={103103},
  year={2023},
  publisher={Elsevier}
}

@book{hardin2002generalized,
  title={Generalized estimating equations},
  author={Hardin, James W and Hilbe, Joseph M},
  year={2002},
  publisher={chapman and hall/CRC}
}

@article{hamilton2025active,
  title={Active Measurement: Efficient Estimation at Scale},
  author={Hamilton, Max and Lai, Jinlin and Zhao, Wenlong and Maji, Subhransu and Sheldon, Daniel},
  journal={arXiv preprint arXiv:2507.01372},
  year={2025}
}

@article{li2025robust,
  title={Robust Sampling for Active Statistical Inference},
  author={Li, Puheng and Zrnic, Tijana and Cand{\`e}s, Emmanuel},
  journal={arXiv preprint arXiv:2511.08991},
  year={2025}
}

@article{crossley2024large,
  title={A large-scale corpus for assessing written argumentation: PERSUADE 2.0},
  author={Crossley, Scott A and Tian, Yu and Baffour, Perpetual and Franklin, Alex and Benner, Meg and Boser, Ulrich},
  journal={Assessing Writing},
  volume={61},
  pages={100865},
  year={2024},
  publisher={Elsevier}
}

@article{freitag2021experts,
  title={Experts, errors, and context: A large-scale study of human evaluation for machine translation},
  author={Freitag, Markus and Foster, George and Grangier, David and Ratnakar, Viresh and Tan, Qijun and Macherey, Wolfgang},
  journal={Transactions of the Association for Computational Linguistics},
  volume={9},
  pages={1460--1474},
  year={2021},
  publisher={MIT Press One Rogers Street, Cambridge, MA 02142-1209, USA journals-info~…}
}

@inproceedings{zhai2023sigmoid,
  title={Sigmoid loss for language image pre-training},
  author={Zhai, Xiaohua and Mustafa, Basil and Kolesnikov, Alexander and Beyer, Lucas},
  booktitle={Proceedings of the IEEE/CVF international conference on computer vision},
  pages={11975--11986},
  year={2023}
}

@article{mozer2025more,
  title={More power to you: Using machine learning to augment human coding for more efficient inference in text-based randomized trials},
  author={Mozer, Reagan and Miratrix, Luke},
  journal={The Annals of Applied Statistics},
  volume={19},
  number={1},
  pages={440--464},
  year={2025},
  publisher={Institute of Mathematical Statistics}
}

@article{chernozhukov2018double,
  title={Double/debiased machine learning for treatment and structural parameters},
  author={Chernozhukov, Victor and Chetverikov, Denis and Demirer, Mert and Duflo, Esther and Hansen, Christian and Newey, Whitney and Robins, James},
  journal={The Econometrics Journal},
  volume={21},
  number={1},
  pages={C1--C68},
  year={2018},
  publisher={Oxford University Press}
}

@article{angelopoulos2023ppi,
  title={Prediction-Powered Inference},
  author={Angelopoulos, Anastasios N. and Bates, Stephen and Fannjiang, Clara and Jordan, Michael I. and Zrnic, Tijana},
  journal={Science},
  volume={382},
  number={6671},
  pages={669--674},
  year={2023},
  publisher={American Association for the Advancement of Science}
}

@article{angelopoulos2023ppipp,
  title={{PPI}++: Efficient Prediction-Powered Inference},
  author={Angelopoulos, Anastasios N. and Duchi, John C. and Zrnic, Tijana},
  journal={arXiv preprint arXiv:2311.01453},
  year={2023}
}

@article{zrnic2024crossppi,
  title={Cross-Prediction-Powered Inference},
  author={Zrnic, Tijana},
  journal={Proceedings of the National Academy of Sciences},
  volume={121},
  number={26},
  pages={e2322083121},
  year={2024},
  publisher={National Academy of Sciences}
}

@book{sarndal1992model,
  title={Model Assisted Survey Sampling},
  author={S{\"a}rndal, Carl-Erik and Swensson, Bengt and Wretman, Jan},
  year={1992},
  publisher={Springer},
  address={New York}
}

@article{cassel1976generalized,
  title={Some Results on Generalized Difference Estimation and Generalized Regression Estimation for Finite Populations},
  author={Cassel, Claes M. and S{\"a}rndal, Carl-Erik and Wretman, Jan H.},
  journal={Biometrika},
  volume={63},
  number={3},
  pages={615--620},
  year={1976},
  publisher={Oxford University Press}
}

@article{zhang2019semisup,
  title={Semi-supervised Inference: General Theory and Estimation of Means},
  author={Zhang, Anru and Brown, Lawrence D. and Cai, T. Tony},
  journal={The Annals of Statistics},
  volume={47},
  number={5},
  pages={2538--2566},
  year={2019},
  publisher={Institute of Mathematical Statistics}
}

@article{cai2018semisup,
  title={Semi-supervised Inference for Explained Variance in High-Dimensional Linear Regression and Its Applications},
  author={Cai, T. Tony and Guo, Zijian},
  journal={Journal of the Royal Statistical Society: Series B},
  volume={82},
  number={2},
  pages={391--419},
  year={2020},
  publisher={Wiley}
}

@inproceedings{zheng2023llmjudge,
  title={Judging {LLM}-as-a-Judge with {MT}-Bench and Chatbot Arena},
  author={Zheng, Lianmin and Chiang, Wei-Lin and Sheng, Ying and Zhuang, Siyuan and Wu, Zhanghao and Zhuang, Yonghao and Lin, Zi and Li, Zhuohan and Li, Dacheng and Xing, Eric P. and Zhang, Hao and Gonzalez, Joseph E. and Stoica, Ion},
  booktitle={Advances in Neural Information Processing Systems},
  volume={36},
  year={2023}
}

@inproceedings{liu2023geval,
  title={{G-Eval}: {NLG} Evaluation using {GPT-4} with Better Human Alignment},
  author={Liu, Yang and Iter, Dan and Xu, Yichong and Wang, Shuohang and Xu, Ruochen and Zhu, Chenguang},
  booktitle={Proceedings of the 2023 Conference on Empirical Methods in Natural Language Processing},
  pages={2511--2522},
  year={2023}
}

@article{wang2023faireval,
  title={Large Language Models are Not Fair Evaluators},
  author={Wang, Peiyi and Li, Lei and Chen, Liang and Zhu, Dawei and Lin, Binghuai and Cao, Yunbo and Liu, Qi and Liu, Tianyu and Sui, Zhifang},
  journal={arXiv preprint arXiv:2305.17926},
  year={2023}
}

@article{li2024llmjudge,
  title={{LLMs}-as-Judges: A Comprehensive Survey on {LLM}-based Evaluation of {NLP}},
  author={Li, Haitao and Dong, Qian and Chen, Junjie and Su, Huawei and Hui, Yiran and Shi, Yue and Fang, Shiyu and Zhu, Xiaohui and Liu, Qingyao and Liu, Yiqun},
  journal={arXiv preprint arXiv:2402.08862},
  year={2024}
}

@article{gu2024llmjudge,
  title={A Survey on {LLM}-as-a-Judge},
  author={Gu, Jiawei and Jiang, Xuhui and Shi, Zhengshan and Tan, Hexiang and Zhai, Xuehao and Xu, Chengjin and Li, Wei and Shen, Lei and Zhu, Shengjie and Cheng, Fei and Ma, Jian},
  journal={arXiv preprint arXiv:2411.15594},
  year={2024}
}

@misc{dubois2024alpacaeval,
  title={{AlpacaEval}: An Automatic Evaluator of Instruction-following Models},
  author={Dubois, Yann and Li, Xuechen and Taori, Rohan and Zhang, Tianyi and Gulrajani, Ishaan and Ba, Jimmy and Guestrin, Carlos and Liang, Percy and Hashimoto, Tatsunori B.},
  year={2024},
  howpublished={\url{https://github.com/tatsu-lab/alpaca_eval}}
}

@inproceedings{du2024multiagentdebate,
  title={Improving Factuality and Reasoning in Language Models through Multiagent Debate},
  author={Du, Yilun and Li, Shuang and Torralba, Antonio and Tenenbaum, Joshua B. and Mordatch, Igor},
  booktitle={Proceedings of the 41st International Conference on Machine Learning},
  year={2024}
}

@inproceedings{liang2024divergent,
  title={Encouraging Divergent Thinking in Large Language Models through Multi-Agent Debate},
  author={Liang, Tian and He, Zhiwei and Jiao, Wenxiang and Wang, Xing and Wang, Yan and Wang, Rui and Yang, Yujiu and Shi, Shuming and Tu, Zhaopeng},
  booktitle={Proceedings of the 2024 Conference on Empirical Methods in Natural Language Processing},
  year={2024}
}

@article{chan2023chateval,
  title={{ChatEval}: Towards Better {LLM}-based Evaluators through Multi-Agent Debate},
  author={Chan, Chi-Min and Chen, Weize and Su, Yusheng and Yu, Jianxuan and Xue, Wei and Zhang, Shanghang and Fu, Jie and Liu, Zhiyuan},
  journal={arXiv preprint arXiv:2308.07201},
  year={2023}
}

@inproceedings{wang2023selfconsistency,
  title={Self-Consistency Improves Chain of Thought Reasoning in Language Models},
  author={Wang, Xuezhi and Wei, Jason and Schuurmans, Dale and Le, Quoc and Chi, Ed and Narang, Sharan and Chowdhery, Aakanksha and Zhou, Denny},
  booktitle={International Conference on Learning Representations},
  year={2023}
}

@article{verga2024jury,
  title={Replacing Judges with Juries: Evaluating {LLM} Generations with a Panel of Diverse Models},
  author={Verga, Pat and Hofst{\"a}tter, Sebastian and Althammer, Sophia and Su, Yixuan and Piktus, Aleksandra and Arkil, Niklas and Fer{\~a}o, Pedro and Razavi, Manzil Zaheer and Saeidi, Sobhi and Ruder, Sebastian and Roth, Patrick},
  journal={arXiv preprint arXiv:2404.18796},
  year={2024}
}

@article{qian2025multijudge,
  title={Enhancing {LLM}-as-a-Judge via Multi-Agent Collaboration},
  author={Qian, Yuqing and Zhang, Shenghua and Zhou, Yupeng and Balakrishnan, Anusha and Li, Jian and Jauhar, Sujay Kumar and Kannan, Anand and Tian, Ran},
  journal={arXiv preprint arXiv:2501.05366},
  year={2025}
}

@inproceedings{zhu2023judgelm,
  title={{JudgeLM}: Fine-tuned Large Language Models are Scalable Judges},
  author={Zhu, Lianghui and Wang, Xinggang and Wang, Xinlong},
  booktitle={International Conference on Learning Representations},
  year={2025}
}

@inproceedings{kim2024prometheus,
  title={Prometheus: Inducing Fine-grained Evaluation Capability in Language Models},
  author={Kim, Seungone and Shin, Jamin and Cho, Yejin and Han, Joel and Cha, Sejune and Lee, Moontae and Seo, Minjoon},
  booktitle={International Conference on Learning Representations},
  year={2024}
}

@article{kim2024prometheus2,
  title={Prometheus 2: An Open Source Language Model Specialized in Evaluating Other Language Models},
  author={Kim, Seungone and Suk, Juyoung and Longpre, Shayne and Lin, Bill Yuchen and Shin, Jamin and Welleck, Sean and Neubig, Graham and Lee, Moontae and Lee, Kyungjae and Seo, Minjoon},
  journal={arXiv preprint arXiv:2405.01535},
  year={2024}
}

@inproceedings{wang2024pandalm,
  title={{PandaLM}: An Automatic Evaluation Benchmark for {LLM} Instruction Tuning Optimization},
  author={Wang, Yidong and Yu, Zhuohao and Zeng, Zhengran and Yang, Linyi and Wang, Cunxiang and Chen, Hao and Jiang, Chaoya and Xie, Rui and Wang, Jindong and Xie, Xing and Ye, Wei and Zhang, Shikun and Zhang, Yue},
  booktitle={International Conference on Learning Representations},
  year={2024}
}

@article{li2023generativejudge,
  title={Generative Judge for Evaluating Alignment},
  author={Li, Junlong and Sun, Shichao and Yuan, Weizhe and Fan, Run-Ze and Zhao, Hai and Liu, Pengfei},
  journal={arXiv preprint arXiv:2310.05470},
  year={2023}
}

@inproceedings{huang2025finetunedjudge,
  title={An Empirical Study of {LLM}-as-a-Judge for {LLM} Evaluation: Fine-tuned Judge Model is not a General Substitute for {GPT-4}},
  author={Huang, Hui and Bu, Xinrun and Zhou, Haoran and Li, Ke and Fan, Xinyu and Guo, Yan and Fei, Hao and He, Qingyu and Li, Jing and Mi, Fei},
  booktitle={Findings of the Association for Computational Linguistics: ACL 2025},
  year={2025}
}

\newpage 
\appendix

\section{Additional Evaluation Results}
\subsection{More on MQM evaluation}\label{app:mqm}
\paragraph{Importance of fitting the appropriate outcome model.}
Table~\ref{tab:mqm-judge-ablation} reports pooled 5-fold OOF metrics for predicting human MQM scores. Features include per-LLM mean scores and variance-based uncertainty, with embeddings from a large text embedding model. Because MQM scores are non-negative with a substantial zero mass ($\sim$13\% pooled, up to 29\% for human translations), we compare the standard Ridge model against a \emph{Hurdle model} (logistic gate predicting $P(\text{MQM}>0)$ combined with Ridge on the positive subset). We also compare two variance encodings: the original $\log(\mathrm{Var}+\epsilon)$ transform versus $\sqrt{\mathrm{Var}}$ with a binary has-variance indicator, since 57--67\% of items have zero LLM variance (all raters agree), creating an artificial bimodal cliff at $\log(10^{-8})\approx -18.4$ in the log encoding.

\begin{table*}[ht!]
    \centering
    \caption{MQM pooled judge-model ablation (5-fold OOF, all 10 systems, $n{=}14{,}180$). The Hurdle model with $\sqrt{\mathrm{Var}}$ encoding achieves the best $R^2$ while eliminating negative predictions.}
    \label{tab:mqm-judge-ablation}
    \begin{tabular}{lcc}
        \toprule
        Configuration & $R^2$ & Spearman \\
        \midrule
        Full (sqrt, hurdle) & \textbf{0.369} & \textbf{0.657} \\
        Full (sqrt, ridge) & 0.355 & 0.646 \\
        Full (log, hurdle) & 0.365 & 0.646 \\
        Full (log, ridge) & 0.351 & 0.634 \\
        \midrule
        LLM scores only (sqrt, ridge) & 0.287 & 0.557 \\
        Embedding-only ridge & 0.167 & 0.497 \\
        Intercept-only baseline & 0.000 & 0.000 \\
        \bottomrule
    \end{tabular}
\end{table*}

The $\sqrt{\mathrm{Var}}$ encoding resolves pathological failures on high-quality systems: for \texttt{Human-B.0} (29\% zeros), Ridge with log-variance yields $R^2=-0.14$ due to extreme predictions driven by the $-18.4$ floor, whereas $\sqrt{\mathrm{Var}}$ with Hurdle achieves $R^2=0.24$. For MT systems with fewer zeros, both encodings perform similarly.

\paragraph{Per-system ablation.}
Table~\ref{tab:mqm-per-system} reports the Hurdle model with $\sqrt{\mathrm{Var}}$ encoding across all ten translation systems and the pooled corpus. The full model (LLM + uncertainty + embeddings) consistently dominates the LLM-only and embedding-only ablations. LLM scores alone explain much of the variance for MT systems ($R^2$ 0.36--0.44) but substantially less for human translations ($R^2$ 0.15--0.21), where inter-rater noise is higher. Embeddings alone provide a complementary signal ($R^2$ 0.10--0.20) that is especially valuable for the human-translation systems where LLM scores are weakest. Combining both sources yields consistent gains across all systems.

\begin{table*}[t]
    \centering
    \caption{Per-system MQM judge-model performance (5-fold OOF, Hurdle model, $\sqrt{\mathrm{Var}}$ encoding). $R^2$ and Spearman~$\rho$ are reported for each feature ablation.}
    \label{tab:mqm-per-system}
    \begin{tabular}{l cc cc cc cc}
        \toprule
        & \multicolumn{2}{c}{Full} & \multicolumn{2}{c}{LLM only} & \multicolumn{2}{c}{Embedding only} & \multicolumn{2}{c}{Intercept only} \\
        \cmidrule(lr){2-3} \cmidrule(lr){4-5} \cmidrule(lr){6-7} \cmidrule(lr){8-9}
        System & $R^2$ & $\rho$ & $R^2$ & $\rho$ & $R^2$ & $\rho$ & $R^2$ & $\rho$ \\
        \midrule
        Human-A.0            & 0.273 & 0.554 & 0.191 & 0.427 & 0.150 & 0.448 & 0.000 & 0.000 \\
        Human-B.0            & 0.236 & 0.516 & 0.150 & 0.359 & 0.139 & 0.452 & 0.000 & 0.000 \\
        Human-P.0            & 0.328 & 0.581 & 0.207 & 0.478 & 0.197 & 0.491 & 0.000 & 0.000 \\
        Huoshan\_Translate    & 0.385 & 0.660 & 0.384 & 0.644 & 0.105 & 0.398 & 0.000 & 0.000 \\
        OPPO                 & 0.425 & 0.691 & 0.388 & 0.649 & 0.160 & 0.470 & 0.000 & 0.000 \\
        Online-A             & 0.441 & 0.703 & 0.391 & 0.691 & 0.112 & 0.375 & 0.000 & 0.000 \\
        Online-B             & 0.369 & 0.646 & 0.399 & 0.641 & 0.116 & 0.365 & 0.000 & 0.000 \\
        Tencent\_Translation  & 0.435 & 0.690 & 0.398 & 0.642 & 0.160 & 0.461 & 0.000 & 0.000 \\
        Tohoku-AIP-NTT       & 0.377 & 0.656 & 0.360 & 0.636 & 0.130 & 0.405 & 0.000 & 0.000 \\
        eTranslation         & 0.449 & 0.707 & 0.436 & 0.673 & 0.130 & 0.441 & 0.000 & 0.000 \\
        \midrule
        Pooled               & 0.369 & 0.657 & 0.287 & 0.557 & 0.167 & 0.497 & 0.000 & 0.000 \\
        \bottomrule
    \end{tabular}
\end{table*}

\subsection{WebDesign evaluation}\label{app:webdesign}
For the Universities subset, we ran the same judge ablation protocol across six outcomes (Aesthetics, Trustworthiness, Typicality, Exemplar Goodness, Family Resemblance, Usability), including a constant intercept-only baseline. Table~\ref{tab:webdesign-unis-ablation} reports average 5-fold OOF metrics across outcomes. The full hybrid model (VLM scores + uncertainty + SigLIP intercept) achieves the best overall performance, improving over both the VLM-only and SigLIP-only ablations. As uncalibrated baselines, using raw VLM scores directly yields highly negative $R^2$ (due to a scale mismatch: human ratings are z-scored while VLM scores are on their original integer scale), though Spearman correlation remains moderate (0.56), confirming that outcome-model calibration is essential in this domain.

\begin{table*}[ht!]
    \centering
    \caption{WebDesign (Universities) judge-model ablation results. Values are averaged across six outcomes. Lower RMSE is better; higher $R^2$ and Spearman are better. Uncalibrated rows use raw VLM scores directly (no outcome-model fitting); note that human ratings are z-scored while VLM scores are on their original integer scale, so scale mismatch explains the negative $R^2$.}
    \label{tab:webdesign-unis-ablation}
    \begin{tabular}{lcccc}
        \toprule
        Configuration & Type & Avg. $R^2$ & Avg. RMSE & Avg. Spearman \\
        \midrule
        VLM scores + uncertainty + SigLIP & Calibrated & 0.555 & 0.462 & 0.727 \\
        VLM scores + uncertainty (no SigLIP) & Calibrated & 0.390 & 0.544 & 0.591 \\
        SigLIP-only & Calibrated & 0.497 & 0.492 & 0.690 \\
        Intercept-only constant & Calibrated & $-$0.001 & 0.699 & $-$0.047 \\
        \midrule
        Raw avg.\ of all judges & Uncalibrated & $-$9.719 & 2.258 & 0.561 \\
        Raw single-judge avg. & Uncalibrated & $-$10.151 & 2.291 & 0.340 \\
        \bottomrule
    \end{tabular}
\end{table*}

\paragraph{Simulation on summary metrics.}
We evaluate each configuration as an AEE outcome model under uniform sampling budgets $\eta \in \{0.05, \ldots, 0.30\}$ across all six perceptual outcomes. In each trial a fraction $\eta$ of websites is sampled uniformly; the AEE estimator combines population-level predictions from the judge model with an in-sample bias correction. Figures~\ref{fig:webdesign-dr-bias} and~\ref{fig:webdesign-dr-outcome} report estimator and outcome-model quality per outcome and on average.

\paragraph{Estimator quality (Figure~\ref{fig:webdesign-dr-bias}).}
All models are approximately unbiased across budgets. The full hybrid model (VLM scores + uncertainty + SigLIP intercept) achieves the lowest standard deviation and RMSE for most outcomes, confirming that the variance-reduction gains from the judge model carry through to the AEE estimator. Perceptual outcomes vary considerably in difficulty: Aesthetics and Usability are better predicted than Exemplar Goodness and Typicality, and this ordering is reflected in estimator efficiency. The AVG panel shows that the full model consistently outperforms both the VLM-only and SigLIP-only ablations when averaged across outcomes.

\begin{figure*}[t]
    \centering
    \begin{subfigure}[b]{0.49\linewidth}
        \includegraphics[width=\linewidth]{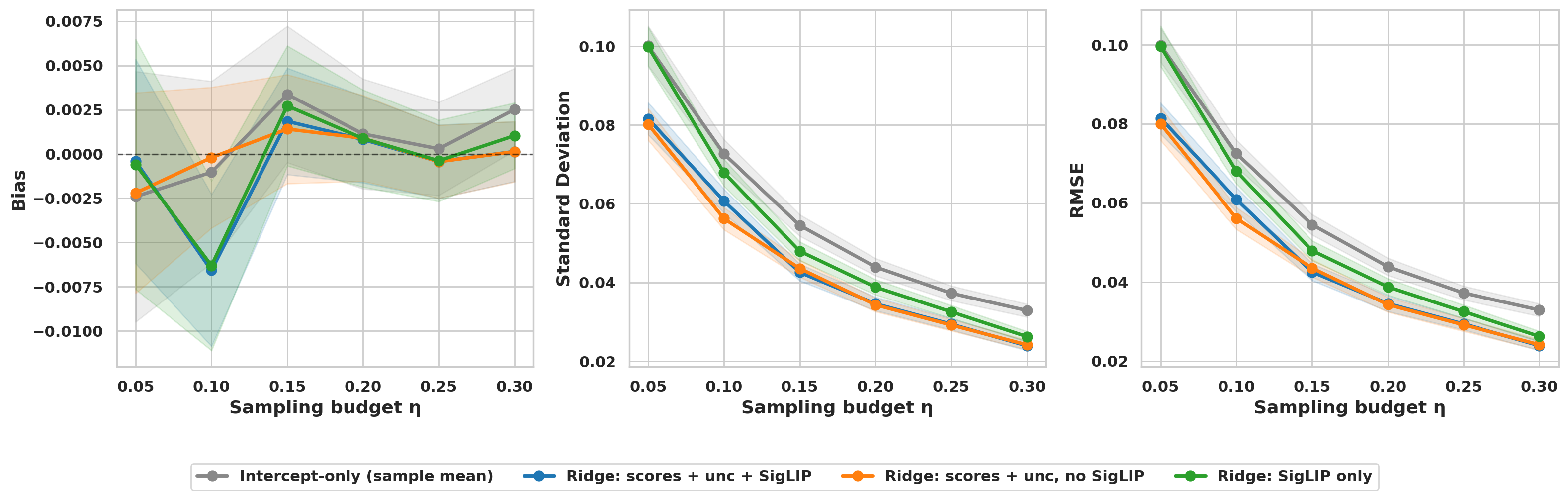}
        \caption{Aesthetics}
    \end{subfigure}\hfill
    \begin{subfigure}[b]{0.49\linewidth}
        \includegraphics[width=\linewidth]{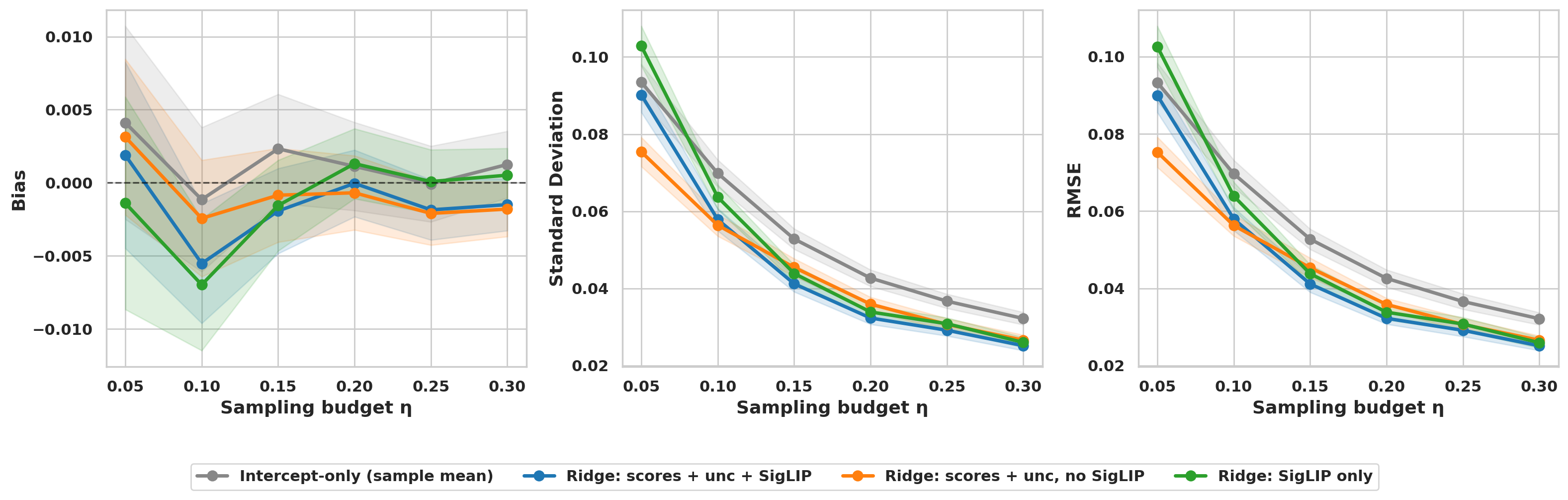}
        \caption{Trustworthiness}
    \end{subfigure}
    \begin{subfigure}[b]{0.49\linewidth}
        \includegraphics[width=\linewidth]{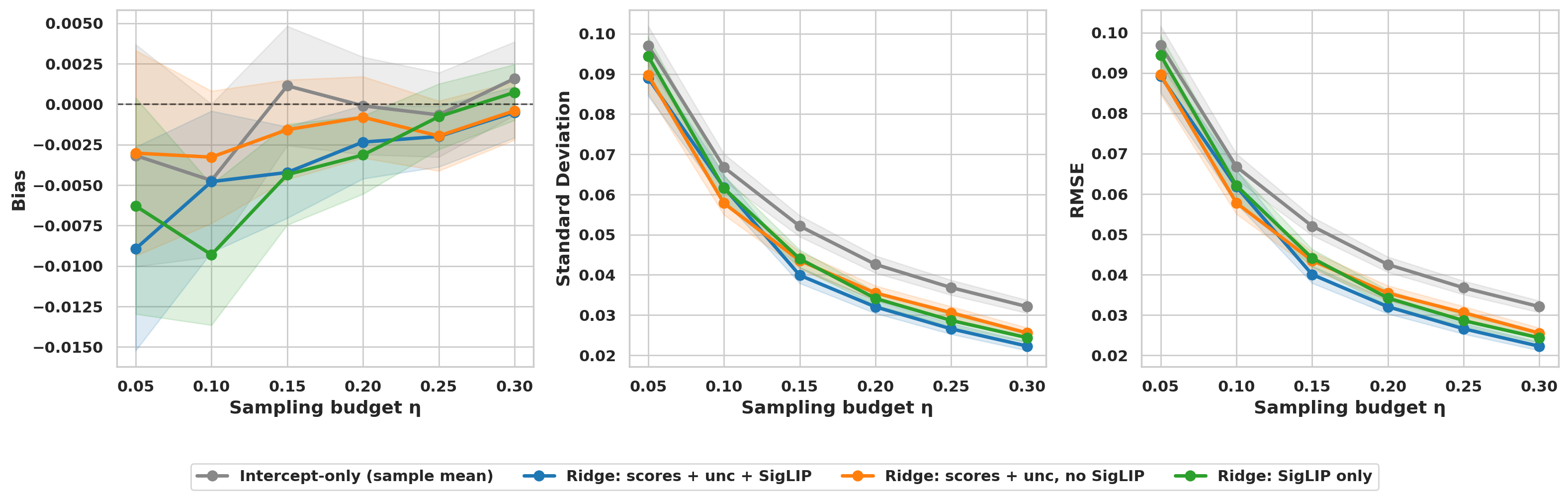}
        \caption{Typicality}
    \end{subfigure}\hfill
    \begin{subfigure}[b]{0.49\linewidth}
        \includegraphics[width=\linewidth]{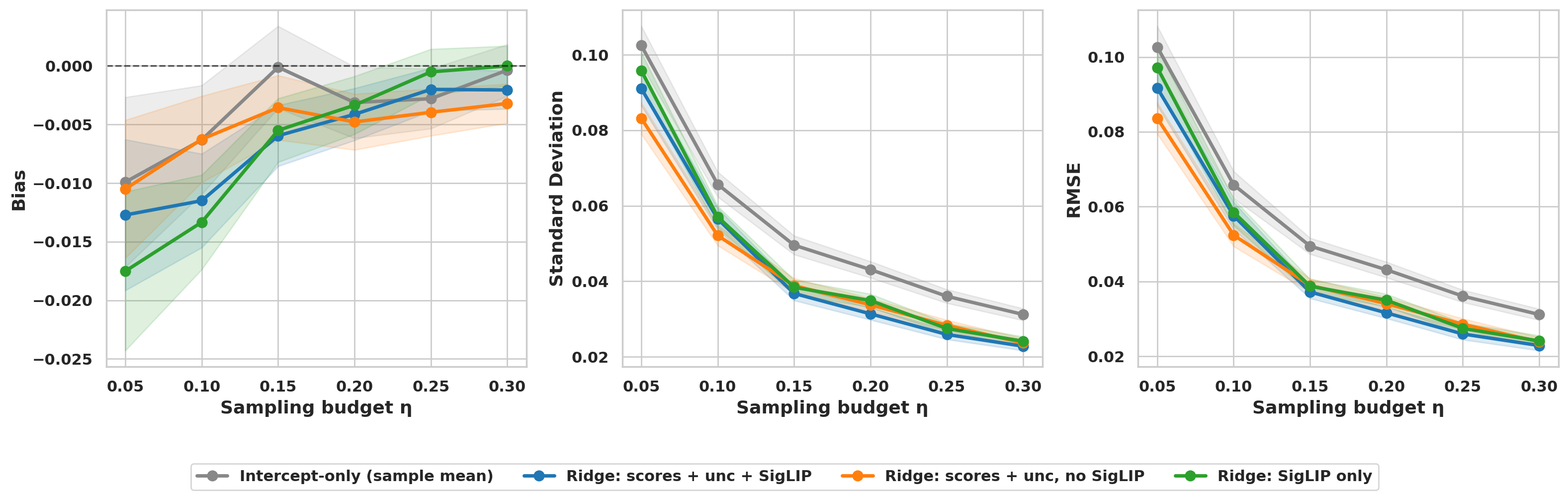}
        \caption{Exemplar Goodness}
    \end{subfigure}
    \begin{subfigure}[b]{0.49\linewidth}
        \includegraphics[width=\linewidth]{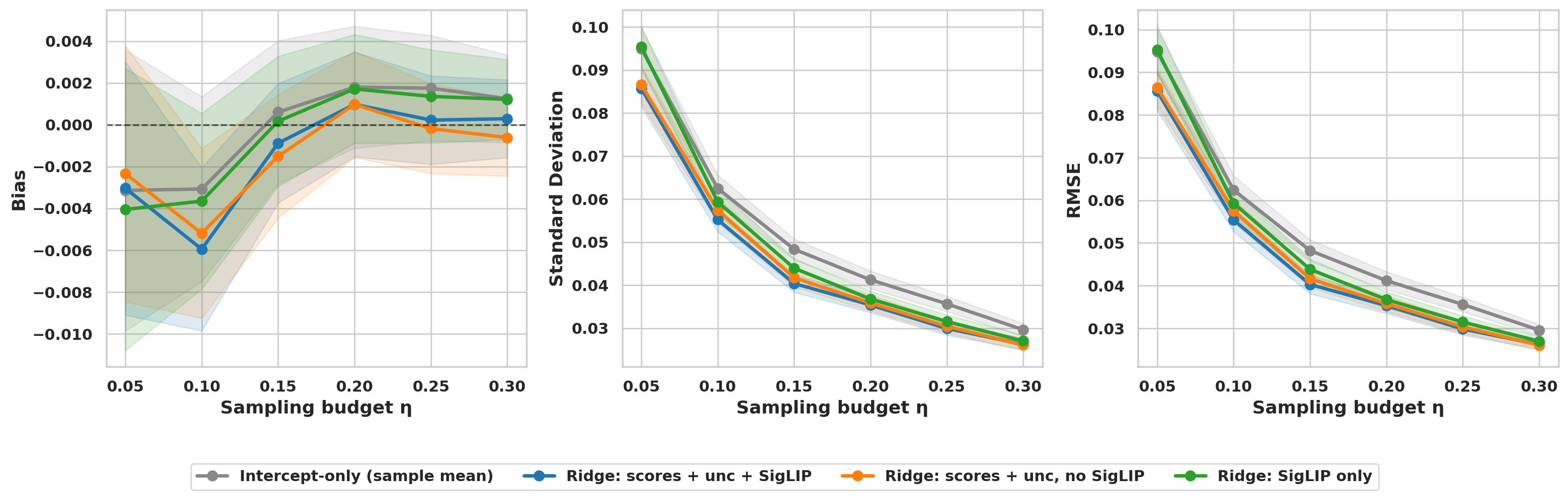}
        \caption{Usability}
    \end{subfigure}\hfill
    \begin{subfigure}[b]{0.49\linewidth}
        \includegraphics[width=\linewidth]{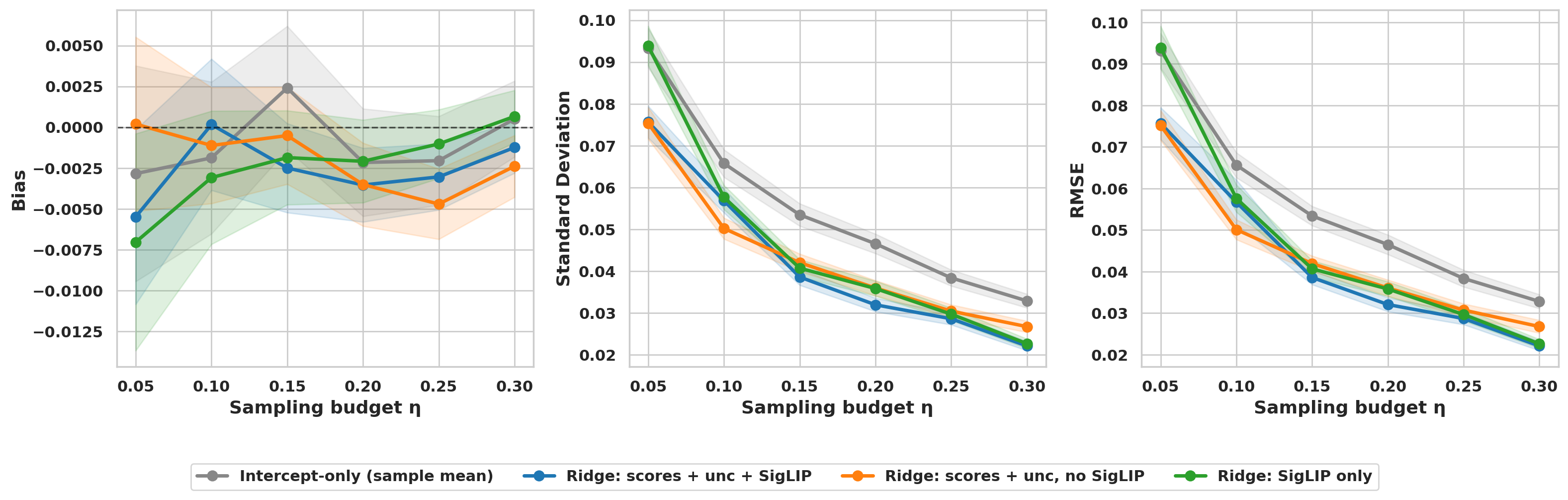}
        \caption{Average (across all outcomes)}
    \end{subfigure}
    \caption{WebDesign (Universities) AEE estimator quality: bias, SD, and RMSE vs.\ sampling budget $\eta$, per perceptual outcome and averaged. Lower is better. The full hybrid model achieves the lowest estimator variance for most outcomes.}
    \label{fig:webdesign-dr-bias}
\end{figure*}

\paragraph{Outcome model quality (Figure~\ref{fig:webdesign-dr-outcome}).}
Figure~\ref{fig:webdesign-dr-outcome} traces the out-of-fold predictive performance of each outcome model as a function of $\eta$. Even at the smallest budget ($\eta{=}0.05$, roughly 30--40 websites), the full model achieves nontrivial $R^2$ and Spearman correlation for several outcomes, demonstrating that VLM judge scores and SigLIP embeddings together provide a strong prior that makes efficient use of limited human labels. Outcome-model quality improves consistently with budget, and the full hybrid model leads on $R^2$ and Spearman for most outcomes, matching the pattern observed in the static ablation (Table~\ref{tab:webdesign-unis-ablation}).

\begin{figure*}[t]
    \centering
    \begin{subfigure}[b]{0.49\linewidth}
        \includegraphics[width=\linewidth]{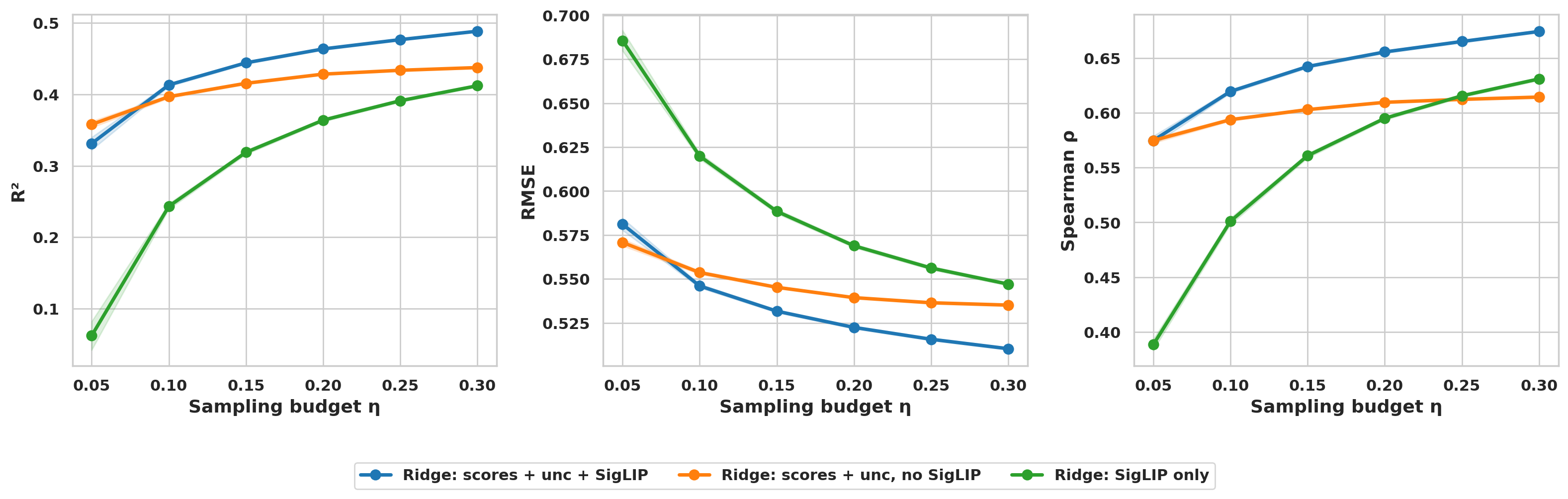}
        \caption{Aesthetics}
    \end{subfigure}\hfill
    \begin{subfigure}[b]{0.49\linewidth}
        \includegraphics[width=\linewidth]{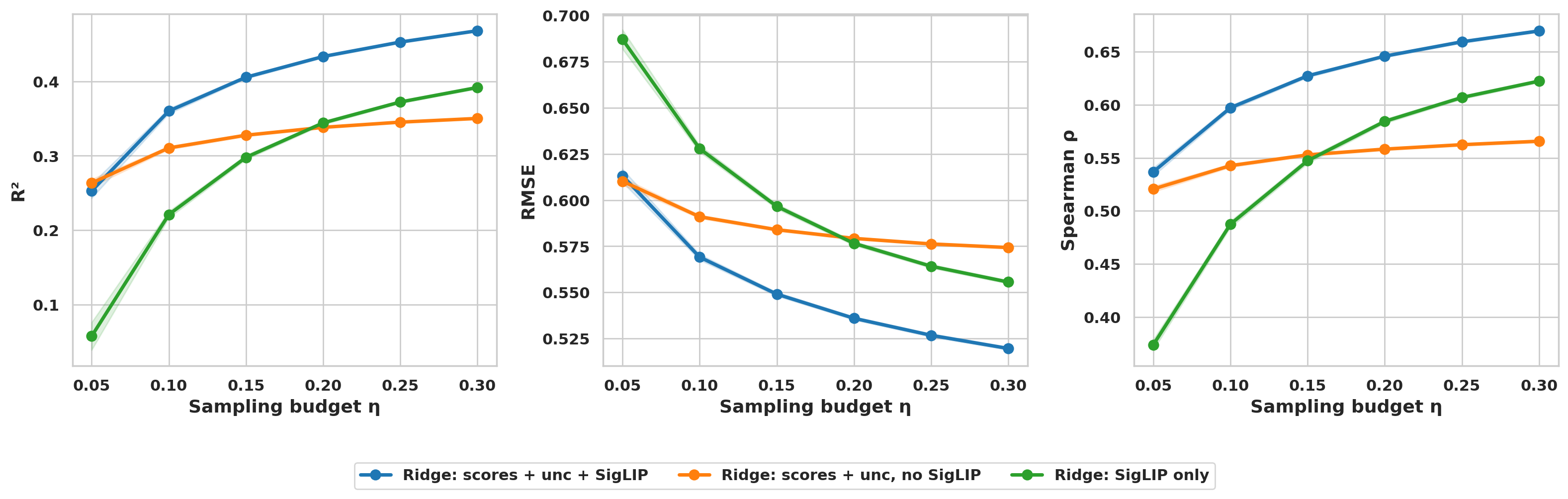}
        \caption{Trustworthiness}
    \end{subfigure}
    \begin{subfigure}[b]{0.49\linewidth}
        \includegraphics[width=\linewidth]{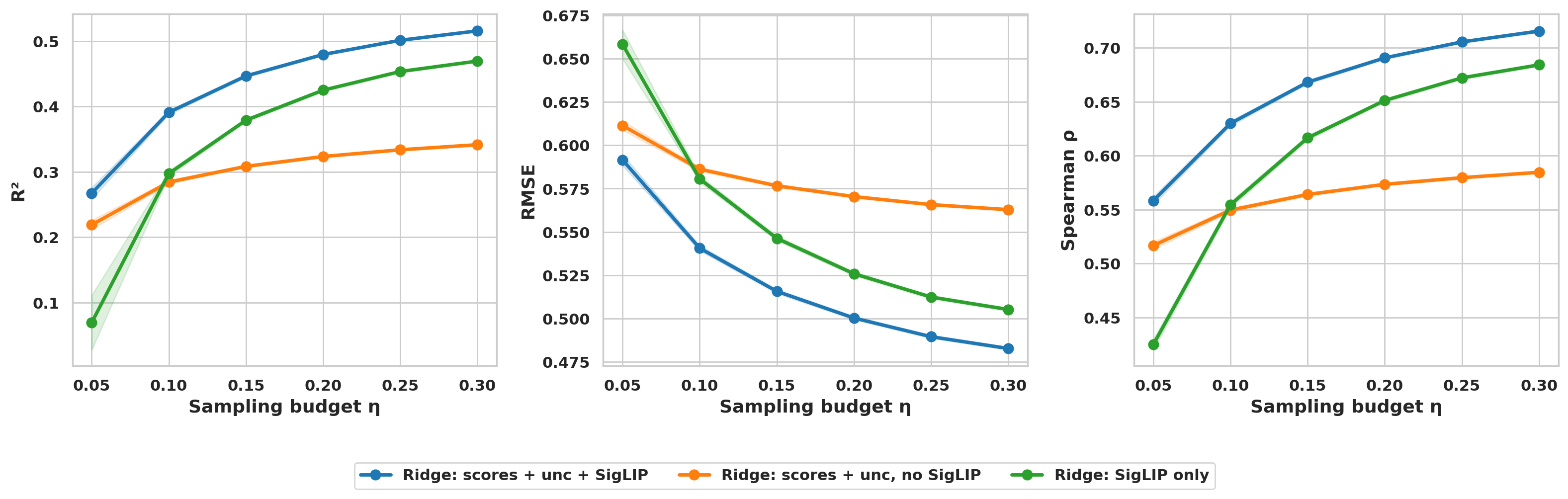}
        \caption{Typicality}
    \end{subfigure}\hfill
    \begin{subfigure}[b]{0.49\linewidth}
        \includegraphics[width=\linewidth]{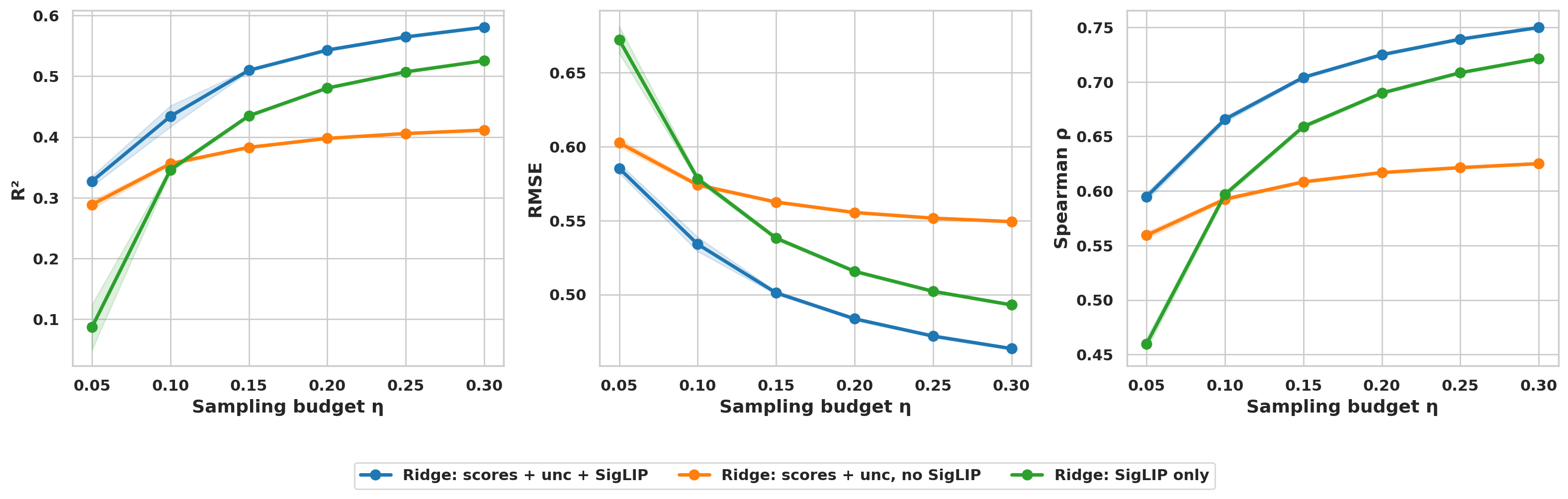}
        \caption{Exemplar Goodness}
    \end{subfigure}
    \begin{subfigure}[b]{0.49\linewidth}
        \includegraphics[width=\linewidth]{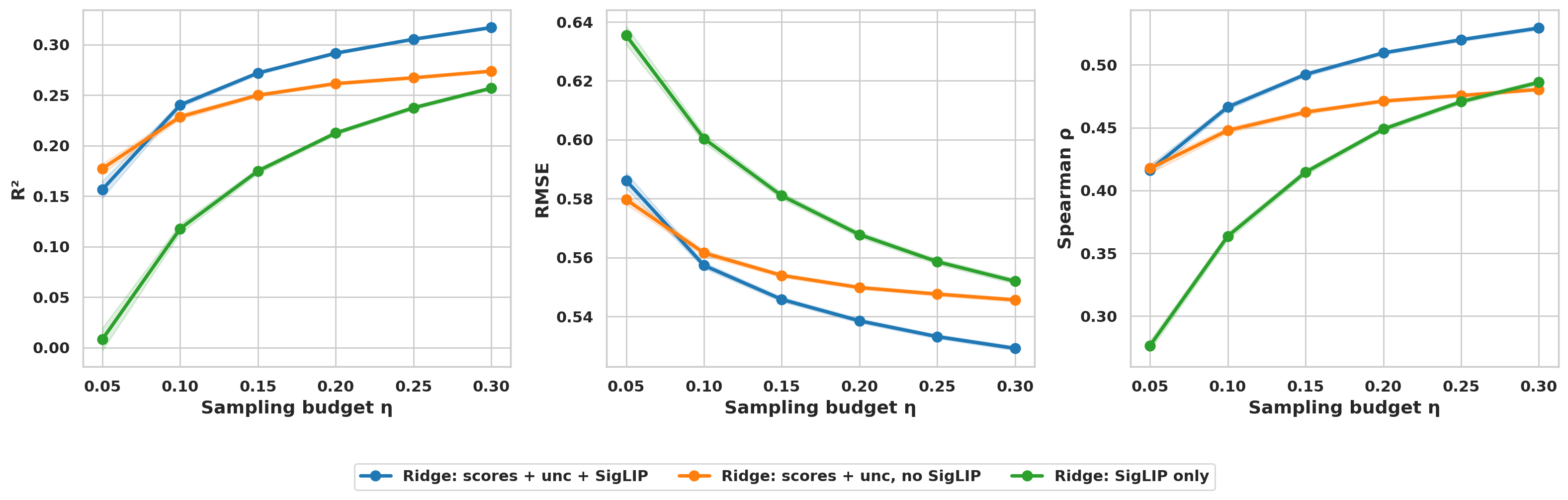}
        \caption{Usability}
    \end{subfigure}\hfill
    \begin{subfigure}[b]{0.49\linewidth}
        \includegraphics[width=\linewidth]{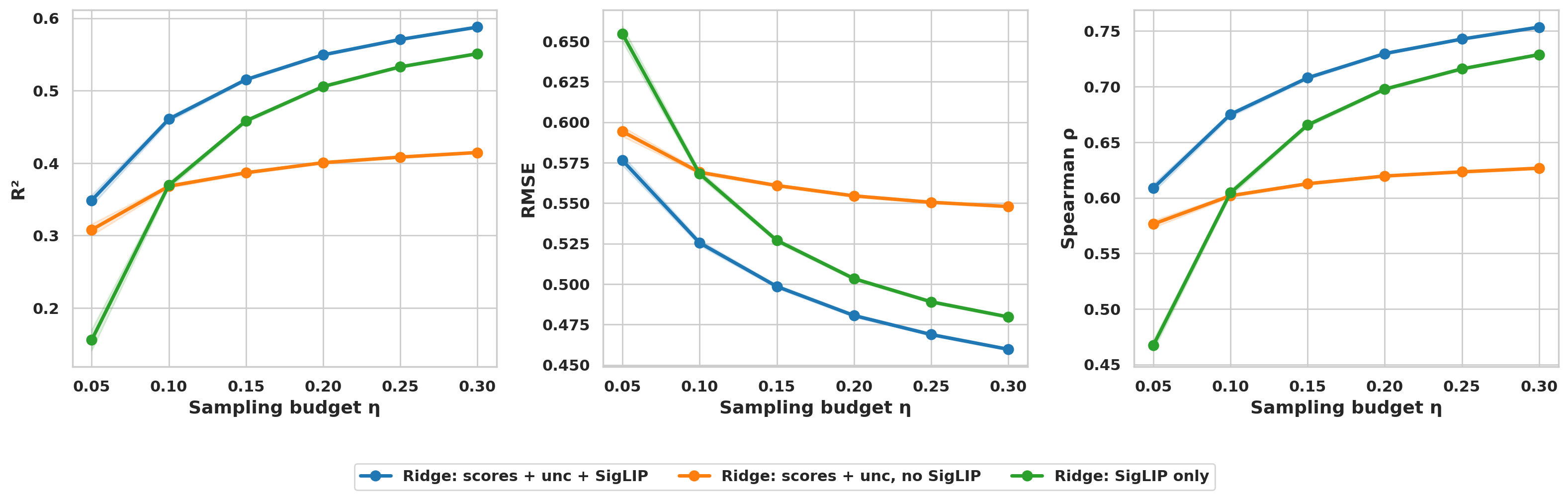}
        \caption{Average (across all outcomes)}
    \end{subfigure}
    \caption{WebDesign (Universities) outcome model quality vs.\ sampling budget $\eta$: out-of-fold $R^2$, RMSE, and Spearman $\rho$ per perceptual outcome and averaged. Higher $R^2$/Spearman and lower RMSE are better. The full model leads on most outcomes.}
    \label{fig:webdesign-dr-outcome}
\end{figure*}

\section{Model-Judge Scoring Protocol}\label{app:prompt}
For each item, we query a panel of language or vision-language judges through a unified API endpoint. Prompts enforce strict scalar outputs to reduce parse ambiguity and improve comparability across models. We use fixed templates:

\textbf{Essay scoring (LLMs).}
\begin{quote}\ttfamily
You are an essay grader. You will be given a student essay and its writing prompt.\\
Rate the essay's holistic quality on a scale of 1 to 6:\\
1 = Very poor, 2 = Poor, 3 = Below average, 4 = Above average, 5 = Good, 6 = Excellent.\\
You MUST respond with ONLY a single digit (1, 2, 3, 4, 5, or 6).\\
No explanation, no reasoning, no other text - just the number.
\end{quote}
\begin{quote}\ttfamily
Essay prompt: \{prompt\_name\}\\
\\
Essay:\\
\{full\_text\}
\end{quote}
Each judge receives this system and user text; we take the first substring matching a digit in $\{1,\dots,6\}$ as the score (same rule as in the scoring script).

\textbf{Web-design scoring (VLMs).}
\begin{quote}\ttfamily
You are a strict visual evaluator. Return only the requested six lines.
\end{quote}
\begin{quote}\ttfamily
Rate this screenshot from a \{domain\_label\} on six dimensions.\\
Return ONLY six integer scores in [-3, -2, -1, 0, 1, 2, 3], one per line, exactly:\\
AE: <score>\\
TRU: <score>\\
TYP: <score>\\
EXMPL: <score>\\
AVG: <score>\\
US: <score>\\
\par
Definitions:\\
- AE (Aesthetics): visually appealing overall design.\\
- TRU (Trustworthiness): appears credible and trustworthy.\\
- TYP (Typicality): \{prompt\_typ\}\\
- EXMPL (Exemplar Goodness): \{prompt\_exmpl\}\\
- AVG (Family Resemblance): \{prompt\_avg\}\\
- US (Usability): easy to navigate and understand.
\end{quote}
The placeholders \{domain\_label\}, \{prompt\_typ\}, \{prompt\_exmpl\}, and \{prompt\_avg\} are filled from a domain table (university/college, commercial-bank, online fashion-shopping, online homeware-shopping) matching our dataset splits. Each judge receives this user text together with one screenshot (system role as above); responses are parsed as six lines of the form \texttt{AE:} $\ldots$, \texttt{TRU:} $\ldots$, etc.

\textbf{Machine translation scoring (LLMs).}
\begin{quote}\ttfamily
You are a professional MQM annotator following the official WMT MQM human evaluation guidelines at the segment level.\\
Identify up to five target-side errors, each with a category (Accuracy, Fluency, Terminology, Style, Locale conventions, or Other, including subcategories such as Mistranslation or Grammar; also Source error or Non-translation when applicable), a severity (Major, Minor, or Neutral), and the exact span from the source or hypothesis.\\
If the translation is too garbled or unrelated to the source, return exactly one error: Non-translation with severity Major.\\
Source errors are listed separately and do not count toward the five target-error limit.\\
Do not double-count the same underlying issue; when multiple issues overlap one span, keep only the most severe; when severities tie, prefer Accuracy $>$ Fluency $>$ Terminology $>$ Style $>$ Locale conventions $>$ Other.\\
Be conservative and return JSON only (no prose outside the object).
\end{quote}
\begin{quote}\ttfamily
[Source]\\
\{source\}\\
\par
[Hypothesis]\\
\{hypothesis\}\\
\par
Return JSON only: a single object with an errors array; each entry has string fields category, severity, and span.
\end{quote}
Each judge receives the source segment and the system hypothesis (machine translation) and returns structured segment-level MQM annotations in that JSON format.

For every API response, we log token-level statistics (including per-token log probabilities and top alternatives when available) and derive uncertainty summaries such as mean token entropy and perplexity. Failed requests are retried with bounded concurrency and backoff; unresolvable failures are marked \texttt{NA} and excluded from successful-completion counts during resume runs.





\section{Additional results}

\subsection{Properties of the estimating equation estimator}\label{app:ee}
In this section, we will present the theoretical properties of the estimating equation estimator.

\begin{theorem}[Consistency and asymptotic normality of the AEE estimator]\label{thm:consistency}
Let $\theta^\star \in \Theta \subset \bbR^p$ be the unique solution to
\[
\Psi(\theta,\gamma^\star) := \E{m(Y,S,X;\theta,\gamma^\star)} = 0.
\]
Assume:
\begin{enumerate}[leftmargin=*]
    \item \textbf{Identification.} $\theta^\star$ is the unique zero of $\Psi(\theta,\gamma^\star)$ in a neighborhood of the truth.
    \item \textbf{Positivity and MAR.} There exists $\epsilon>0$ such that $\pi^\star(X)\ge \epsilon$ almost surely, and $S \perp Y \mid X$.
    \item \textbf{Smoothness.} The map $\theta \mapsto m(Y,S,X;\theta,\gamma)$ is continuously differentiable near $\theta^\star$, and
    \[
    A := \E{\nabla_\theta m(Y,S,X;\theta^\star,\gamma^\star)}
    \]
    exists and is nonsingular.
    \item \textbf{Nuisance estimation.} $\|\hat{\gamma}-\gamma^\star\|=o_p(1)$ and
    \[
    \mathbb{E}_N\!\left[m(Y,S,X;\theta^\star,\hat{\gamma})\right]
    -
    \mathbb{E}_N\!\left[m(Y,S,X;\theta^\star,\gamma^\star)\right]
    =
    o_p(N^{-1/2}).
    \]
    This holds, for example, under Neyman orthogonality together with standard consistency rates for the nuisance estimators.
    \item \textbf{Empirical process control.} A uniform law of large numbers and a central limit theorem hold for the relevant class of score functions in a neighborhood of $(\theta^\star,\gamma^\star)$.
\end{enumerate}
If $\hat{\theta}$ satisfies
\[
\mathbb{E}_N\!\left[m(Y,S,X;\hat{\theta},\hat{\gamma})\right] = o_p(N^{-1/2}),
\]
then $\hat{\theta}\overset{p}{\to}\theta^\star$. Moreover,
\[
\sqrt{N}(\hat{\theta}-\theta^\star)
=
-A^{-1}\frac{1}{\sqrt{N}}\sum_{i=1}^N m(Y_i,S_i,X_i;\theta^\star,\gamma^\star) + o_p(1),
\]
and hence
\[
\sqrt{N}(\hat{\theta}-\theta^\star)\rightsquigarrow \mathcal{N}(0,\Sigma_\theta),
\qquad
\Sigma_\theta = A^{-1} B A^{-T},
\]
where
\[
B = \E{m(Y,S,X;\theta^\star,\gamma^\star)m(Y,S,X;\theta^\star,\gamma^\star)^\top}.
\]
\end{theorem}

\begin{corollary}[Consistency of the sandwich variance estimator]
If, in addition, $\hat{A}\overset{p}{\to}A$ and $\hat{B}\overset{p}{\to}B$, then
\[
\widehat{\Var{\hat{\theta}}}=\frac{1}{n}\hat{A}^{-1}\hat{B}\hat{A}^{-T}
\]
is a consistent estimator of the asymptotic covariance of $\hat{\theta}$.
\end{corollary}

\begin{remark}
This theorem applies directly to the augmented estimating equations for the population mean and quantiles introduced above. In those cases, the nuisance functions are the labeling propensity $\pi(X)$ together with the regression adjustments $\mu(X)$ or $\phi(X;\theta)$. Neyman orthogonality is especially useful because it makes the estimator insensitive, to first order, to regularization bias from flexible nuisance fits.
\end{remark}

\subsection{Proof of Theorem \ref{thm:consistency}}
\begin{proof}
Consistency follows from standard $Z$-estimation arguments. By the uniform law of large numbers and nuisance consistency,
\[
\sup_{\theta \in \Theta}
\left\|
\mathbb{E}_N[m(Y,S,X;\theta,\hat{\gamma})]
-
\Psi(\theta,\gamma^\star)
\right\| = o_p(1).
\]
Since $\Psi(\theta,\gamma^\star)$ has a unique zero at $\theta^\star$, any approximate empirical root must converge in probability to $\theta^\star$.

For asymptotic normality, expand the empirical moment around $(\theta^\star,\gamma^\star)$:
\[
0
=
\mathbb{E}_N[m(Y,S,X;\hat{\theta},\hat{\gamma})]
=
\mathbb{E}_N[m(Y,S,X;\theta^\star,\gamma^\star)]
+ A(\hat{\theta}-\theta^\star)
+ r_N,
\]
where $r_n=o_p(n^{-1/2})$ by differentiability, stochastic equicontinuity, and the assumed negligible first-order effect of nuisance estimation. Rearranging gives
\[
\sqrt{N}(\hat{\theta}-\theta^\star)
=
-A^{-1}\sqrt{N}\,\mathbb{E}_N[m(Y,S,X;\theta^\star,\gamma^\star)] + o_p(1).
\]
The multivariate central limit theorem yields
\[
\sqrt{N}\,\mathbb{E}_N[m(Y,S,X;\theta^\star,\gamma^\star)]
\rightsquigarrow
\mathcal{N}(0,B),
\]
and Slutsky's theorem implies the stated Gaussian limit with covariance $A^{-1}BA^{-T}$.
\end{proof}

\section{Extension: constrained Weighting Optimization for Adaptive Sampling}\label{app:sampling}

\subsection{Framework}
Suppose we can learn a good outcome model from the one round of labeling, we can then solve a constrained weighting objective to optimize future sampling schema under a fixed budget when the goal is to estimate the population summary metrics. We use population means as an example. Formally, we want to solve the following problem:
\[
\min_{\sigma} \sum_i \frac{r_i^2}{e(x_i)} \quad \text{subject to} \quad \frac{1}{N}\sum_i e(x_i) \le \eta,
\]
where $e(x_i) = \pi + (1-\pi)\sigma(b + x_i^\top \beta)$ is the inclusion probability for item $i$, and $b$ and $\beta$ are the parameters of the outcome model.
We use an augmented Lagrangian approach with a hinge-squared penalty for constraint violations. This formulation encourages larger weights on high-residual items while enforcing a global sampling-rate cap.

We propose the following finite-population labeling workflow:
\begin{enumerate}
    \item Draw an initial random sample (stage 1), collect human labels, and fit an outcome model $f$.
    \item Fit a propensity model $\sigma$ from features (primarily embedding PCs) to approximate hardness-informed inclusion scores.
    \item Sample additional items adaptively (stage 2) using probabilities derived from $\sigma$ under an overall budget $\eta$.
    \item Estimate the population mean using the AEE estimator that combines inverse-propensity correction with outcome-model augmentation.
\end{enumerate}

This objective prioritizes high-residual items while enforcing a global sampling-rate cap. This is relevant to active statistical inference literature: \citep{li2025robust, hamilton2025active}.

\subsection{Variance characterization of the AEE estimator}
First, let's think about the variance of the AEE estimator. For simplicity, let's assume that the outcomes are i.i.d. draws from a super population $(y, x)$ from some distribution $P$. Then, we can write the AEE estimator as
\begin{align*}
    \hat{\mu} = \frac{1}{N} \sum_i \frac{S_i (y_i - f(x_i))}{e(x_i)} + f(x_i).
\end{align*}
Due to independence, the variance of $\hat{\mu}$ is given by
\begin{align*}
    \frac{1}{N} \Var{\frac{S (y - f(x))}{e(x)} + f(x)}.
\end{align*}
Now we compute the variance term. First, notice the following decomposition:
\begin{align*}
  & \frac{S (y - f(x))}{e(x)} + f(x)\\
   = & \underbrace{\frac{S (y - f^\star(x))}{e(x)}}_{\text{Term I}} + \underbrace{f^\star(x)}_{\text{Term II}} + \underbrace{\frac{(S - e(x))(f^\star(x) - f(x))}{e(x)}}_{\text{Term III}},
\end{align*}
Note that Term I, II and III are mutually orthogonal, i.e.
\begin{align*}
\Cov{\text{Term I}}{\text{Term II}} =& 0 \\
\Cov{\text{Term I}}{\text{Term III}} =& 0 \\
\Cov{\text{Term II}}{\text{Term III}} =& 0.
\end{align*}

Therefore, the variance of $\hat{\mu}$ is given by
\begin{align*}
    \frac{1}{N} \Var{\text{Term I}} + \frac{1}{N} \Var{\text{Term II}} + \frac{1}{N} \Var{\text{Term III}}.
\end{align*}
Now we compute the variance of each term.
\begin{align*}
    \Var{\text{Term I}} = \E{\frac{\sigma^{*2}(x)}{e(x)}}.
\end{align*}
For Term II, we have
\begin{align*}
    \Var{\text{Term II}} = \E{(f^\star(x) - \mu)^2}.
\end{align*}
For Term III, we have
\begin{align*}
    \Var{\text{Term III}} = \E{\frac{(1 - e(x))(f^\star(x) - f(x))^2}{e(x)}}.
\end{align*}

Therefore we summarize two forms of the variance of the AEE estimator:

\begin{itemize}
    \item \textbf{Form 1:}
    \begin{align*}
        \Var{\hat{\mu}} = \frac{1}{N}\E{\frac{\sigma^{*2}(x)}{e(x)} + (f^\star(x) - \mu)^2 + \frac{(1 - e(x))(f^\star(x) - f(x))^2}{e(x)}}.
    \end{align*}
    \item \textbf{Form 2:}
    \begin{align*}
        \Var{\hat{\mu}} = \frac{1}{N}\E{\frac{(y - f(x))^2}{e(x)} + (f^\star(x) - \mu)^2 - (f^\star(x) - f(x))^2}.
    \end{align*}
\end{itemize}

\paragraph{What are we actually optimizing when we optimize the propensity score?}

From Form 1, we can see that given a fixed propensity model, we can benefit from fitting a better outcome model to reduce the variance of the AEE estimator.

From Form 2, we can see that given a fixed outcome model, we can benefit from fitting a better propensity model to reduce the variance of the AEE estimator. More importantly, the best choice of propensity model should be proportional to the square root of the conditional quadratic error of the outcome model. To show this, let's try to optimize the propensity model based on Form 2. Define the conditional squared error as:
\begin{align*}
    \sigma^2_f(x) = \E{(y - f(x))^2 \mid x}.
\end{align*}
When $f = f^\star$, we have $\sigma^2_f(x) = \sigma^{*2}(x)$.
By Cauchy-Schwarz inequality, we have
\begin{align*}
    \E{\sigma_f(x)}^2 \le \E{\frac{\sigma^2_f(x)}{e(x)}} \cdot \E{e(x)} \le  \eta \E{\frac{\sigma^2_f(x)}{e(x)}}.
\end{align*}
Therefore, we have the following lower bound:
\begin{align*}
    \E{\frac{\sigma^2_f(x)}{e(x)}} \ge \frac{\E{\sigma_f(x)}^2}{\eta}.
\end{align*}
The equality holds when
\begin{align*}
    e(x) \propto \sigma_f(x).
\end{align*}
That is, the best choice of propensity model should be proportional to the square root of the conditional quadratic error of the outcome model. This is consistent with the intuition that we should sample more items that are harder to predict.

\subsection{Synthetic experiments on summary metrics}

To isolate the benefit of adaptive sampling under controlled heteroskedasticity, we construct a synthetic labeling environment that inherits the covariate structure of the real PERSUADE corpus but replaces human scores with outcomes drawn from a known data-generating process (DGP).

\subsubsection{Data-Generating Process}
Let $e_i\in\mathbb{R}^d$ denote the OpenAI text embedding for essay $i$, and let $\bar\ell_i$ be the average LLM-judge score across the model panel. We compute principal components from the standardised embedding matrix and define the conditional mean
\[
\mu_i \;=\; \sum_{j=1}^{5}\mathrm{PC}_{ij} \;+\; \bar\ell_i,
\]
so that the outcome depends on both intrinsic essay features (through the first five PCs) and the aggregate model-judge signal. Conditional variance is heteroskedastic in the first principal component:
\[
\sigma_i^2 \;=\; 0.1\,\exp \bigl(\lambda\,\mathrm{PC}_{i1}\bigr),
\]
with $\lambda>0$ governing the severity of heteroskedasticity. For each item we draw
\[
Y_i \;\sim\; \mathcal{N}\!\bigl(\mu_i,\;\sigma_i^2\bigr).
\]
This DGP makes high-PC1 items intrinsically noisier while low-PC1 items remain relatively clean, creating a directional heteroskedasticity that the propensity model can exploit by selectively oversampling the noisy tail.

\subsubsection{Simulation Protocol}
We sweep $\lambda\in\{0.5, 1.0, 1.5, 2.0, 2.5, 3.0\}$ to trace the effect of increasing heteroskedasticity. For each value of $\lambda$, the synthetic outcomes are generated once and held fixed across Monte Carlo trials. Each trial then executes one of three estimation strategies under a common total labeling budget $\eta$:

\begin{enumerate}
    \item \textbf{Baseline 1 (Direct Sample Mean).} Draw $\lfloor\eta N\rfloor$ items uniformly at random, collect their labels, and report the sample mean.
    \item \textbf{Baseline 2 (One-Stage estimation).} Draw $\lfloor\eta N\rfloor$ items uniformly, fit an outcome model $\hat f$ by ridge regression on embedding PCs, and form an AEE estimator with constant propensity $e(x_i)=\eta$.
    \item \textbf{Two-Stage Adaptive estimation.} Draw a stage-1 pilot sample at rate $\pi<\eta$. Fit an outcome model $\hat f$ and compute squared residuals $r_i^2=(Y_i-\hat f(x_i))^2$ on the pilot. Fit a propensity model $\sigma(x_i)$ on embedding PCs to predict hardness, then draw stage-2 items with inclusion probability $e(x_i)=\pi+(1-\pi)\sigma(x_i)$ subject to a mean-budget constraint $\frac{1}{N}\sum_i e(x_i)\le\eta$. Combine both stages via an AEE estimator with item-specific propensities $e(x_i)$.
\end{enumerate}

The propensity model solves the constrained weighting problem described above using an augmented Lagrangian optimizer, where $e(x_i)=\pi+(1-\pi)\,\sigma(b+x_i^\top\beta)$ and $(b,\beta)$ are chosen to minimize $\sum_i r_i^2/e(x_i)$ subject to $\frac{1}{N}\sum_i e(x_i)\le\eta$. We allow separate PCA dimensionalities for the outcome model and the propensity model.

Performance is assessed by the root mean squared error (RMSE) of each estimator relative to the true population mean $\mu=N^{-1}\sum_i\mu_i$, computed over repeated Monte Carlo trials. We additionally report bias and standard deviation to diagnose whether improvements stem from variance reduction or bias correction.

\subsubsection{Synthetic Simulation Results}

Figure~\ref{fig:synthetic-bias-var-mse} reports the bias, variance, and MSE of all three estimators as $\lambda$ increases. The direct sample mean (Baseline~1) maintains near-zero bias but its variance stays flat and high across all $\lambda$ values, as uniform sampling cannot adapt to the heteroskedastic noise structure. The one-stage AEE estimator (Baseline~2) substantially reduces variance relative to the direct mean at low $\lambda$ by leveraging the outcome model, but its variance grows steadily with $\lambda$ as residual heteroskedasticity overwhelms the constant-propensity correction. The two-stage adaptive AEE estimator achieves the lowest MSE across all $\lambda$ values. Its advantage widens as heteroskedasticity increases: at $\lambda=3.0$ the two-stage MSE is roughly half that of the one-stage AEE estimator and over an order of magnitude below the direct mean. All three estimators remain approximately unbiased throughout the sweep, confirming that the MSE gains are driven entirely by variance reduction through adaptive allocation of the labeling budget toward the noisy tail of the PC1 distribution.

\begin{figure*}[t]
    \centering
    \includegraphics[width=0.95\linewidth]{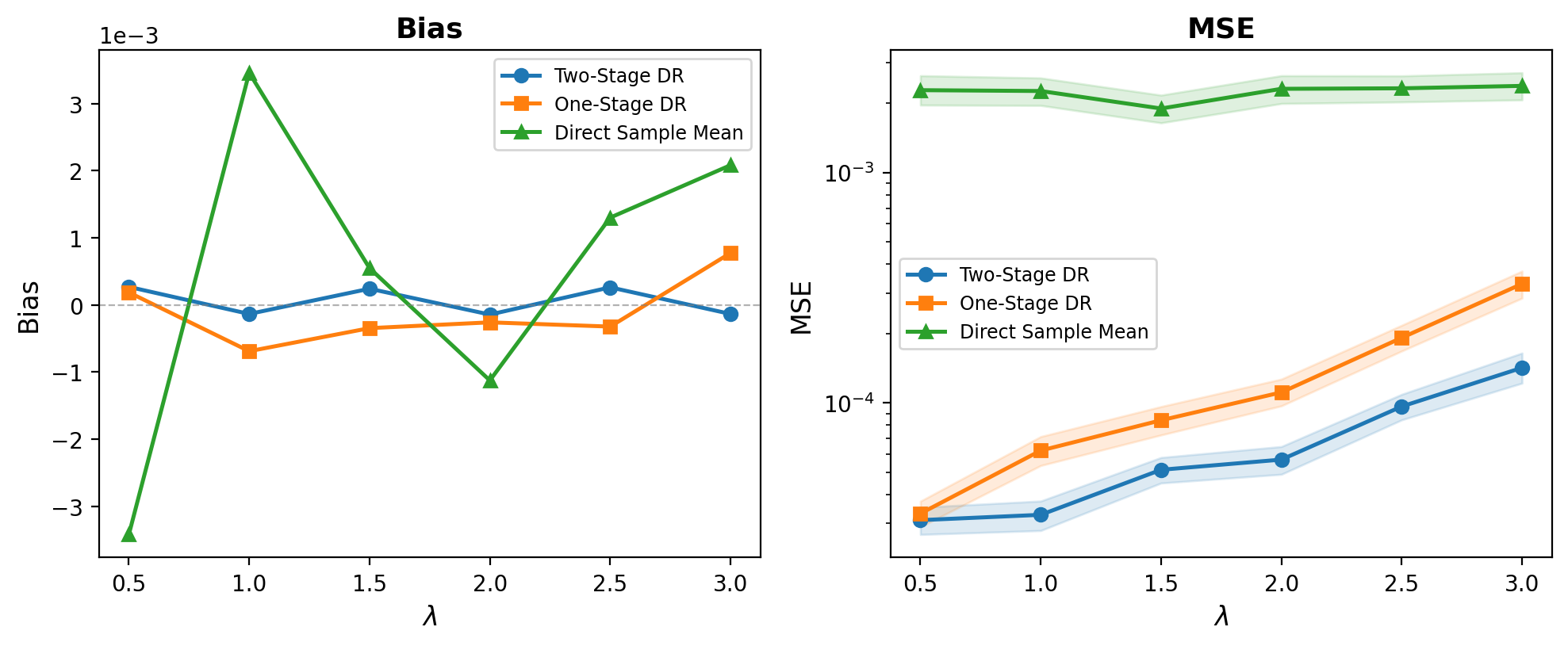}
    \caption{Synthetic simulation: bias and MSE of three population-mean estimators as a function of the heteroskedasticity parameter~$\lambda$. Budget $\eta=0.40$, pilot rate $\pi=0.15$, 400 Monte Carlo trials per $\lambda$ value.}
    \label{fig:synthetic-bias-var-mse}
\end{figure*}

\subsubsection{Sensitivity to Pilot Rate and Budget}

We next fix $\lambda=1.5$ and study how estimation quality depends on the pilot rate~$\pi$ and the total budget~$\eta$.

\paragraph{Varying $\pi$ (fixed $\eta=0.40$).}
Figure~\ref{fig:sweep-pi} sweeps $\pi\in\{0.05,0.10,\ldots,0.30\}$. Because the baselines use the full budget $\eta$ in a single uniform draw, their MSE is invariant to~$\pi$. The two-stage AEE estimator is relatively stable across this range, with a slight optimum near $\pi=0.15$; very small pilots ($\pi=0.05$) leave the outcome model under-trained, while large pilots ($\pi\ge0.25$) consume too much budget before the adaptive stage can act. All estimators remain approximately unbiased.

\paragraph{Varying $\eta$ (fixed $\pi=0.15$).}
Figure~\ref{fig:sweep-eta} sweeps $\eta\in\{0.20,0.25,\ldots,0.45\}$. As expected, MSE decreases monotonically with budget for all three methods. The two-stage AEE estimator maintains a consistent advantage over both baselines at every budget level: even at the tightest budget ($\eta=0.20$) it achieves lower MSE than the direct mean at $\eta=0.45$. The relative improvement of two-stage AEE estimator over one-stage AEE estimator is largest at small~$\eta$, where the adaptive allocation of the scarce remaining budget after the pilot is most valuable.

\begin{figure*}[t]
    \centering
    \includegraphics[width=0.95\linewidth]{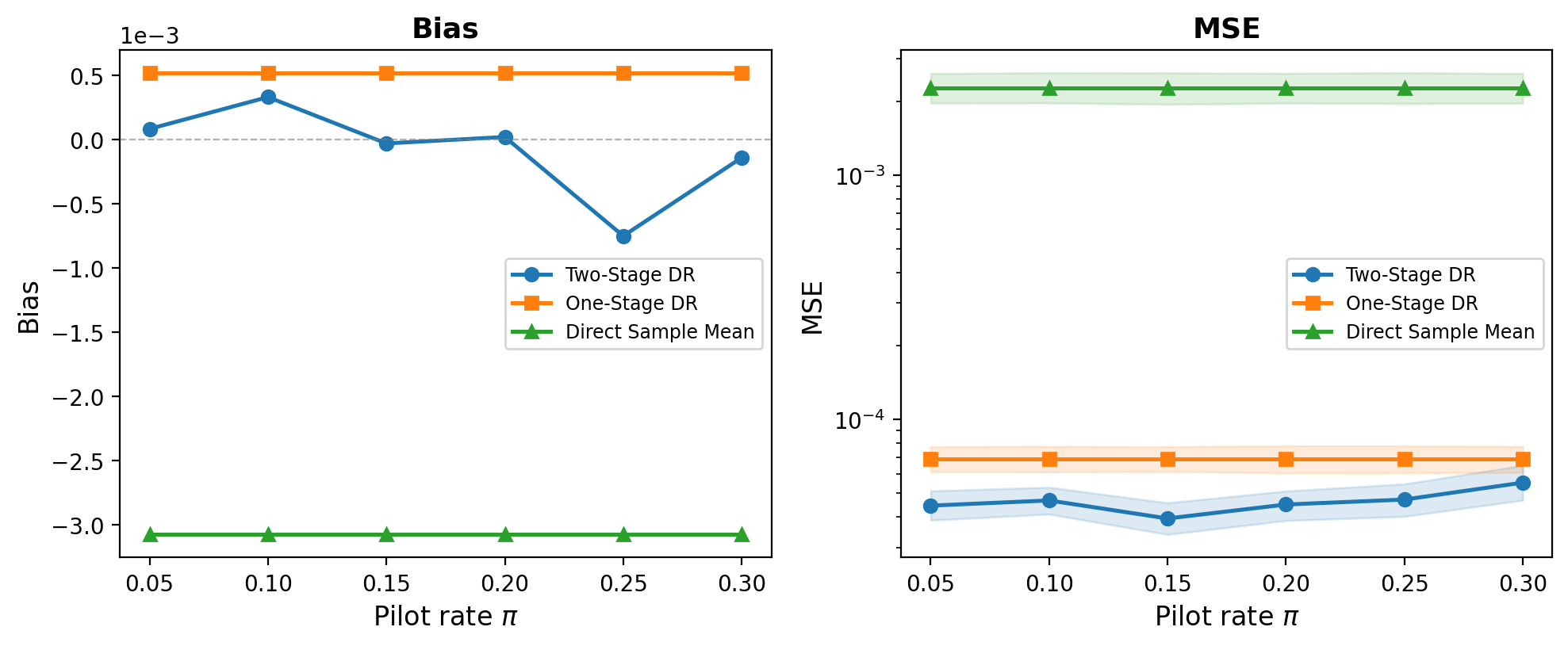}
    \caption{Sensitivity to pilot rate~$\pi$ with fixed budget $\eta=0.40$ and $\lambda=1.5$. Left: bias; right: MSE (log scale). 400 trials per setting.}
    \label{fig:sweep-pi}
\end{figure*}

\begin{figure*}[t]
    \centering
    \includegraphics[width=0.95\linewidth]{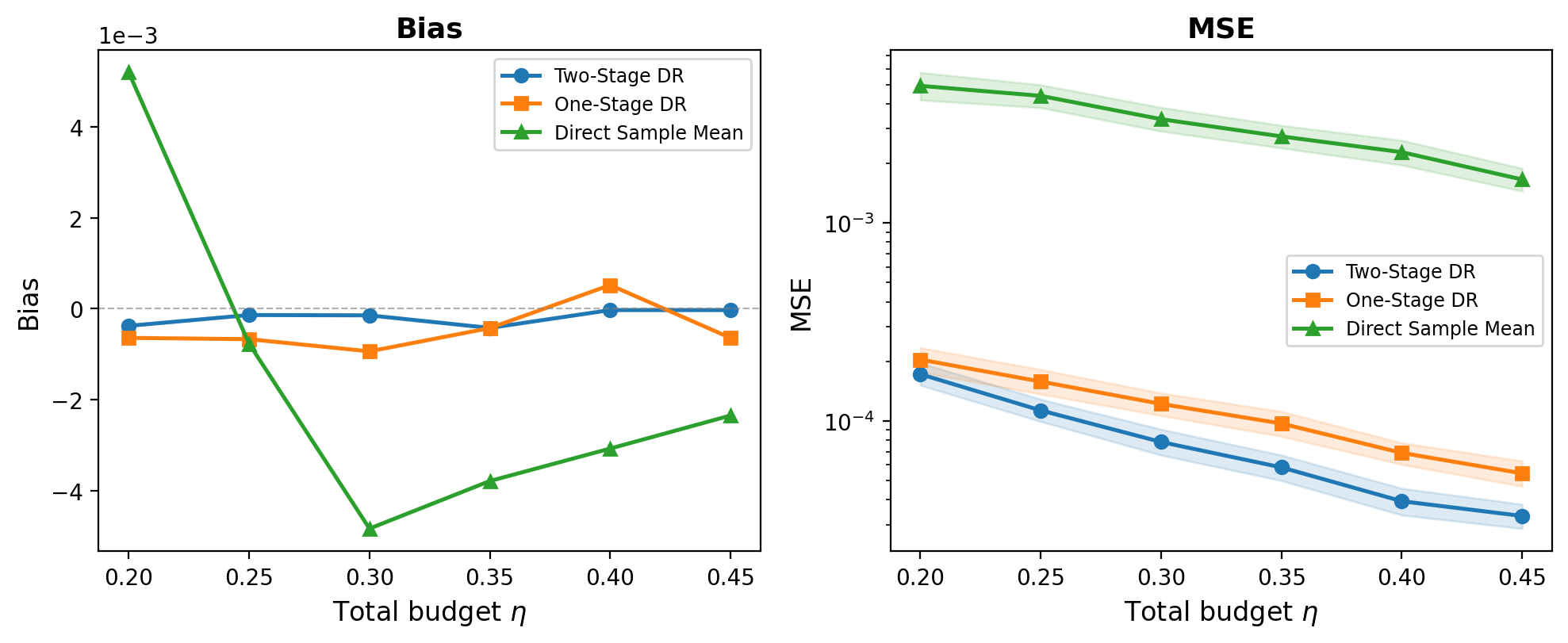}
    \caption{Sensitivity to total budget~$\eta$ with fixed pilot rate $\pi=0.15$ and $\lambda=1.5$. Left: bias; right: MSE (log scale). 400 trials per setting.}
    \label{fig:sweep-eta}
\end{figure*}

\newpage

\end{document}